\documentclass{article}

\usepackage[preprint]{neurips_2023}

\PassOptionsToPackage{}{natbib}
\usepackage[hidelinks,colorlinks=true,linkcolor=blue,citecolor=blue]{hyperref}

\usepackage[utf8]{inputenc} 
\usepackage[T1]{fontenc}    
\usepackage{hyperref}       
\usepackage{url}            
\usepackage{booktabs}       
\usepackage{amsfonts}       
\usepackage{nicefrac}       
\usepackage{microtype}      
\usepackage{xcolor}         

\usepackage[ruled,vlined]{algorithm2e}
\usepackage{dsfont}
\usepackage{amsmath}
\usepackage{amsthm}
\usepackage{amssymb}
\usepackage{color} 
\definecolor{Prune}{RGB}{99,0,60}
\usepackage{mdframed}
\usepackage{multirow} 
\usepackage{multicol} 
\usepackage{scrextend} 
\usepackage{tikz}
\usepackage{graphicx}
\usepackage[absolute]{textpos} 
\usepackage{colortbl}
\usepackage{array}
\usepackage{subfigure}
\usepackage{mathrsfs} 
\usepackage{stmaryrd}

\newcommand\alphat{\Tilde{\alpha}}

\newcommand\Xtest{\boldsymbol{X}_{n+1}}
\newcommand\X{\boldsymbol{X}}

\newcommand\x{\boldsymbol{x}}
\newcommand\Xd{\boldsymbol{X}^\diamond}

\newcommand\cov{\text{cov}}

\newcommand\hV{\widehat{V}}
\newcommand{\labs}{\left\lvert}
\newcommand{\rabs}{\right\rvert}

\theoremstyle{plain}
\newtheorem{theorem}{Theorem}[section]

\newtheorem{lemma}[theorem]{Lemma}

\theoremstyle{definition}

\newtheorem{assumption}[theorem]{Assumption}
\theoremstyle{remark}

\title{Adaptive Conformal Prediction by Reweighting Nonconformity Scores}

\author{%
  Salim I.~Amoukou \\
  LaMME\\
  University Paris Saclay\\
  Stellantis Paris\\
 \And
   Nicolas J-B.~Brunel \\
  LaMME  \\
  ENSIIE, University Paris Saclay\\
  Quantmetry Paris \\
}

\begin{document}

\maketitle

\begin{abstract}
Despite attractive theoretical guarantees and practical successes, Predictive Interval (PI) given by Conformal Prediction (CP) may not reflect the uncertainty of a given model. This limitation arises from CP methods using a constant correction for all test points, disregarding their individual uncertainties, to ensure coverage properties. To address this issue, we propose using a Quantile Regression Forest (QRF) to learn the distribution of nonconformity scores and utilizing the QRF's weights to assign more importance to samples with residuals similar to the test point. This approach results in PI lengths that are more aligned with the model's uncertainty. In addition, 
the weights learnt by the QRF provide a partition of the features space, allowing for more efficient computations and improved adaptiveness of the PI through groupwise conformalization. Our approach enjoys an assumption-free finite sample marginal and training-conditional coverage, and under suitable assumptions, it also ensures conditional coverage. Our methods work for any nonconformity score and are available as a \href{https://github.com/salimamoukou/ACPI}{Python package}. We conduct experiments on simulated and real-world data that show significant improvements compared to existing methods.
\end{abstract}

\addtocontents{toc}{\protect\setcounter{tocdepth}{0}}
\section{Motivations}
Machine learning techniques offer single point predictions, such as mean estimates for regression and class labels for classification, without providing any indication of uncertainty or reliability. This can be a major concern in high-stakes applications where precision is vital.

Consider a training set $\small \mathcal{D}_m = \{(\X_i, Y_i)\}_{i=1}^{m}$ with $\small (\X_i, Y_i) \in \mathcal{X} \times \mathbb{R}$ drawn exchangeably from $\small P= P_{\X} P_{Y \vert \X}$, and an algorithm $\small \mathcal{A}$ that gives conditional mean or quantile estimate $\small \mathcal{A}(\mathcal{D}_m) = \widehat{\mu}(\cdot)$. We consider the problem of constructing a predictive set $C(\cdot)$ for the unseen response $\small Y_{n+1}$ given a new feature $\small \Xtest$. Conformal Prediction is a universal framework that constructs a prediction interval $\small C(\Xtest)$ that covert $\small Y_{n+1}$ with finite-sample (non-asymptotic) coverage guarantee without any assumption on $\small P$ and $\small \widehat{\mu}$. CP methods can be broadly divided into two categories: those that involve retraining the model multiple times, such as full conformal \citep{vovk2005algorithmic} or jackknife methods \citep{barber2021predictive}, and those that use sample splitting, known as split conformal methods \citep{Papadopoulos2002InductiveCM, Lei2016DistributionFreePI}. The latter is more computationally feasible at the cost of splitting the data. In this paper, we consider the split conformal approach (split-CP).

The foundation of the PI of the CP framework is the nonconformity score $\small \hV(\X, Y)$ that represents the error of the model $\small \widehat{\mu}$ on $(\X, Y)$. Given a calibration set $\small \mathcal{D}_n = \{(\X_i, Y_i)\}_{i=1}^{n}$ independent of the training set $\small \mathcal{D}_m = \{ (\X_i, Y_i)\}_{i=1}^{m}$, and the scores $\small \hV_i := \hV(\X_i, Y_i)$ for all $i \in \mathcal{D}_n$, the PI of $\small \Xtest$ at level $1-\alpha$ given by the split-CP is:
\begin{align} \label{chap7:eq:predictive_set}
    \small C(\Xtest)= \left\{y\in \mathbb{R}:  \hV(\Xtest,y)\leq \mathcal{Q}\left(1-\alpha; \; \widehat{F}_{n+1}\right) \right\},
\end{align}
where $\small \mathcal{Q}(1-\alpha; \;F)$ denotes the $\small (1-\alpha)$-quantile of any cumulative distribution function (c.d.f) $F$, and $\small \widehat{F}_{n+1}(\cdot)$ is the empirical c.d.f of the samples $\small \hV_{1:n} \cup \infty$ defined as $\small \widehat{F}_{n+1}(r) = \sum_{i=1}^n \frac{1}{n+1} \mathds{1}_{\hV_i\leq r} + \frac{1}{n+1} \mathds{1}_{\infty \leq r}$.
By exchangeability of the $n+1$ data points $ (\X_1, Y_1),\dots, (\X_n, Y_n), (\X_{n+1}, Y_{n+1}) $, we can show that the PI has marginal coverage, i.e.,
\begin{align*}
    \small P^{n+1}\left(Y_{n+1} \in C(\Xtest)\right) \geq 1 - \alpha.
\end{align*}
$P^{n+1}$ denotes that the probability is taken with respect to the $\small n+1$ data points and $\small \alpha \in (0, 1)$ is a predefined miscoverage rate. However, despite the marginal guarantees, split-CP cannot represent the variability of the model's uncertainty given $\small \Xtest$. Indeed, it constructs the PI of future test points $\small \Xtest$ through the uniform distribution over the calibration residuals $\small \widehat{F}_{n+1}(\cdot)$  that treat all the calibration residuals as the same regardless of $\small \Xtest$. To better illustrate the issue, consider a simple example where the true distribution of $Y$ is homoskedastic, meaning that $\small Y = \mu(\X) + \epsilon$, where $\small \X$ and $\small \epsilon$ are independent.  In this case, the true residuals of the calibration samples $\small V_i:= V(\X_i, Y_i) = |Y_i - \mu(\X_i)| = |\epsilon|$ are independent of $\small \X_i$ and  $\small V_i \sim |\epsilon|$ for $i \in \mathcal{D}_n$. Hence, we have $\small F_{V}(\cdot) = F_{V|\X=\x}(\cdot)$. However, in practice, we only have the estimated residuals, $\small \hV_i:= \hV(\X_i, Y_i) = |Y_i - \widehat{\mu}(\X_i)| = |\mu(\X_i) - \widehat{\mu}(\X_i) + \epsilon|$, which do depend on $\small \X_i$ as the accuracy of $\small \widehat{\mu}$ can vary for different $\X_i$. For example, if $\small \X_i$ is in a high density region with a large amount of data, $\small \widehat{\mu}$ is likely to be more accurate, while in a low density region with a small amount of data, $\small \widehat{\mu}$ is likely to be less accurate. In contrast of the true residual, the conditional law of the estimated residuals $\small \hV |\X = \x$ is not equal to the marginal law of $\small \hV$, thus using the latter $\small F_{\hV}(\cdot)$ as in split-CP to construct the PI of a given observation $\small \x$ may produce under/over coverage PI as $\small \mathcal{Q}\left(1-\alpha;\; F_{\hV}\right)$ may be greater or lower than $\small \mathcal{Q}\left(1-\alpha;\; F_{\hV|\X = \x}\right)$.

Our goal is to construct Prediction Intervals (PIs) with valid coverage for the model of interest $\small \widehat{\mu}$, while adjusting the width of the intervals to help visualize and understand the source of uncertainty of the model $\small \widehat{\mu}$. In fact, the split-CP uses a constant correction term $\small \mathcal{Q}\left(1-\alpha; \; \widehat{F}_{n+1}\right)$ for all test samples, while we aim to have an adaptive correction term that depends on the specific test observation $\Xtest$. To achieve this, we propose to directly estimate the conditional distribution of the nonconformity score given $\X_{n+1}$ by re-weighting the distribution $\small \widehat{F}_{n+1}(\cdot)$ in order to favor the residuals $\{\hV_i\}_{i \in \mathcal{D}_n}$ closer to the residual of $\X_{n+1}$. In Figure \ref{chap7:fig:first_figure}, we show the correction terms of split-CP, our method, and the true error of the model $\small \widehat{\mu}$ computed on california house price dataset \citep{california_data}. It shows that our corrections are more aligned with the true error of the model. 

\begin{figure}[ht!]
    \centering
    \subfigure{\includegraphics[width=0.3\textwidth]{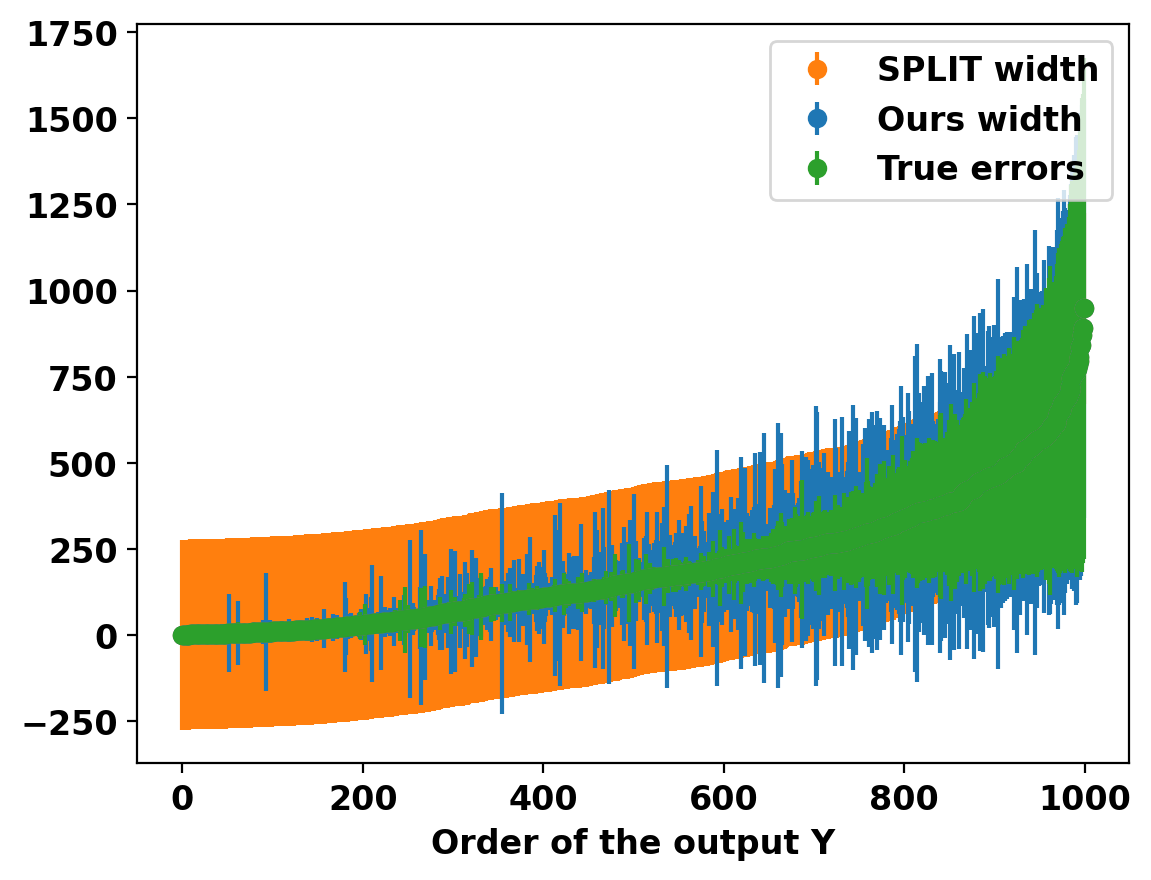}}
\subfigure{\includegraphics[width=0.3\textwidth]{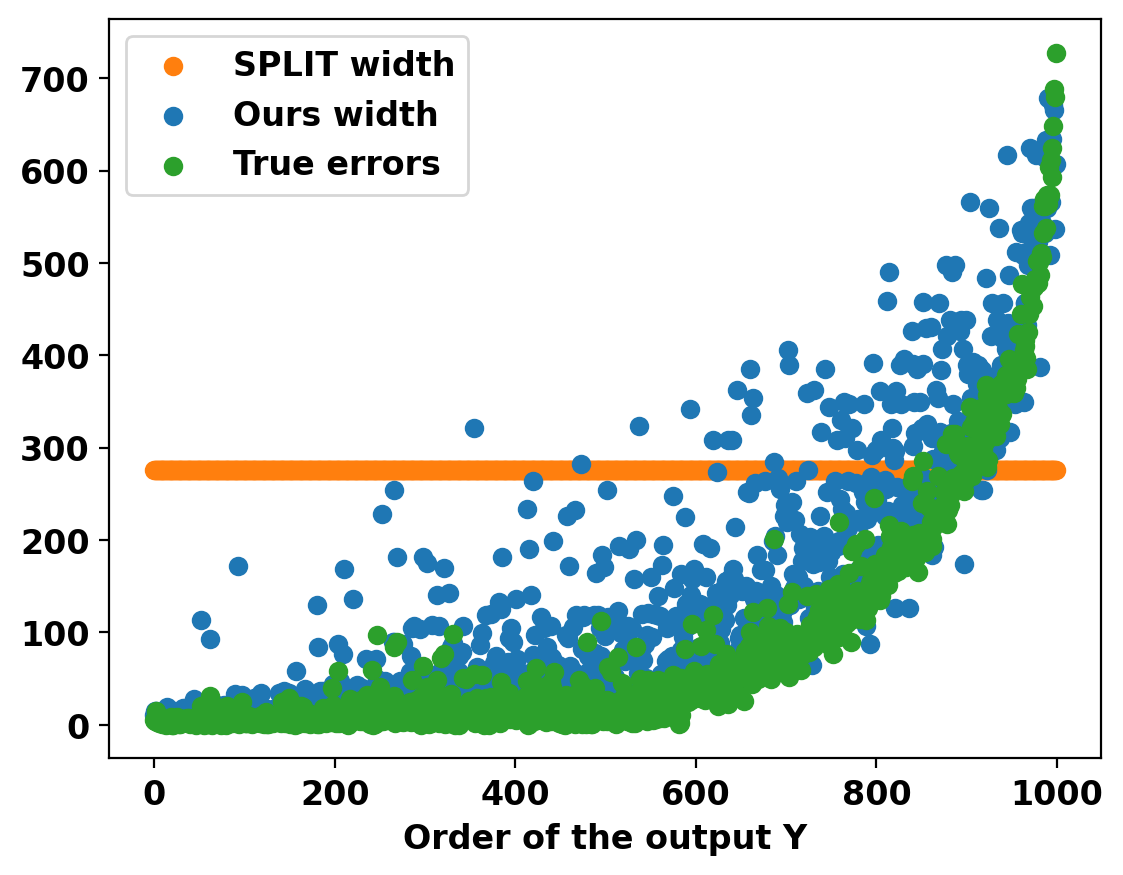}}
    \caption{Correction terms of SPLIT, Ours, and the true error on california house price  dataset}
    \label{chap7:fig:first_figure}
\end{figure}

We also aim to give PI with stronger coverage guarantee. Indeed, in practical applications, what is of interest is the coverage rate on future test points based on a given calibration set. However, the marginal coverage does not address this concern. It only bounds the coverage rate on average over all possible sets of calibration and test observations. In contrast, the training-conditional coverage ensures that with  probability $1-\delta$ over the calibration samples $\mathcal{D}_n$, the resulting coverage on future test observation is still above $1 - \alpha$. Formally, 
\begin{align*}
     \small P^n \left( P\left(Y_{n+1} \in C(X_{n+1})  \;|\;\mathcal{D}_n\right) \geq 1 - \alpha \right) \geq 1 - \delta.
\end{align*}
This style of guarantee is also known as  “Probably Approximately Correct” (PAC) predictive interval \citep{valiant1984theory}. The roots of this type of guarantee can be traced back to the earlier works of \citep{wilks1941determination,wald1943extension}. Despite the importance of training-conditional coverage in practice, only a few methods have been proven to achieve it. \cite{vovk2012conditional} was the first to establish this result for split conformal methods, and recently \cite{bian2022training} has shown that the K-fold CV+ method also achieves it. However, no analogous results are currently known for other CP methods, such as jackknife+ \citep{barber2021predictive} and full-conformal \citep{vovk2005algorithmic}. Therefore, we propose a further calibration step such that our proposed adaptive PI also achieves training-conditional coverage.

There is another area of research that focuses on developing CP procedures for conditional coverage $\mathbb{P}(Y_{n+1} \in C(X_{n+1}) \; \vert \; \X_{n+1}=\x ) \geq 1-\alpha $. It is well known that obtaining nontrivial distribution-free conditional coverage is impossible with a finite sample \citep{Lei2014DistributionfreePB, vovk2012conditional}. Consequently, we prove under suitable assumptions that our methods also achieve asymptotic conditional coverage.

\section{Related works and contributions}

For the sake of simplicity, we use the absolute residual as the nonconformity score $\small \hV_i := \hV(\X_i, Y_i) = |Y_i - \widehat{\mu}(\X_i)|$, without loss of generality.  As a result, the best (symmetric) PI that can be constructed with  $\small \widehat{\mu}(\cdot)$ and the score $\small \hV(\cdot)$ is $\small C^\star(\X_{n+1}) = \big[\widehat{\mu}(\X_{n+1}) \pm q^\star_{1-\alpha}(\Xtest)\big]$ where $\small q^\star_{1-\alpha}(\Xtest)$ is the $\small (1 - \alpha)$-quantile of $\small F_{\hV_{n+1} | \Xtest}$.  To construct adaptive PIs, we propose focusing on the estimated residuals of the calibration samples $\small \{ \hV_i\}_{i \in \mathcal{D}_n}$, and approximate the distribution of $\small \hV | \X = \x$ or identify the stable regions $A$ where $\small Var(\hV(\X, Y) \;|\;\X \in A) \approx 0$, which would allow us to isolate the regions where there is high/low uncertainty of the model.

Recently, \cite{guanlocalizer} proposed Localized Conformal Prediction (LCP) and \cite{han2022split} inspired by \citep{lin2021locally} proposed Split Localized Conformal Prediction (SLCP) which uses kernel-based weights $\small w_h(\X_i, x)$ or Nadaraya-Watson (NW) estimator \citep{nadaraya1964estimating} to approximate the conditional c.d.f of $\small\hV | \X = \x$. Both methods differ in how they learn the NW estimator, SLCP uses the training data $\small\mathcal{D}_m$ to learn the estimator $\small\widehat{F}^{(S)}_h(r | \X = \x) = \sum_{i \in \mathcal{D}_m} w_h(\x, \X_i) \mathds{1}_{\hV_i \leq r}$, while LCP uses the calibration data $\small\mathcal{D}_n$ to learn the estimator $\small\widehat{F}^{(L)}_h(r | \X = \x) = \sum_{i \in \mathcal{D}_n} w_h(\x, \X_i) \mathds{1}_{\hV_i \leq r}$. The calibration step of these two methods is also different. The PI of $\Xtest$ given by SLCP is:
\begin{align*}
   \small C^{S}(\Xtest) = \Big[\widehat{\mu}(\Xtest) \pm \mathcal{Q}\left(1-\alpha; \; \widehat{F}^{(S)}_h(\cdot |\X = \Xtest)\right) + \textcolor{red}{\widehat{Q}} \Big]
\end{align*}
where $\small \textcolor{red}{\widehat{Q}}$ is the split-CP correction term to achieve marginal coverage. In constrast, LCP does not use split-CP but instead adapts the threshold $\small \textcolor{red}{\alphat} = 1 - \alpha$ in $\small \mathcal{Q}\left(1-\alpha;\; \widehat{F}^{(L)}_h(\cdot |\Xtest)\right)$ to achieve the marginal coverage. LCP constructs the predictive interval for a new point $\Xtest$ as follows: 
\begin{align*}
     \small C^{L}(\Xtest) = \Big[\widehat{\mu}(\Xtest) \pm \mathcal{Q}\left(\textcolor{red}{\alphat}; \; \widehat{F}^{(L)}_h(\cdot | \X = \Xtest)\right)\Big]
\end{align*}
where $\textcolor{red}{\alphat}$ is chosen to achieve the marginal coverage.
However, while both LCP and SLCP address the problem and guarantee marginal coverage, they have some limitations. A main limitation is that they are based on kernel methods, which are known to be limited in high dimensions due to the curse of dimensionality. Additionally, choosing the appropriate kernel bandwidth can be challenging and it can be difficult to define kernels that handle both categorical and continuous variables. Another limitation of SLCP is that it learns $\small \widehat{F}^{(S)}_h(\cdot | \X = \x)$ on the training data $\small \mathcal{D}_m$, which may result in overfitting and thus the calibration step using split-CP may produce large intervals to attain the marginal coverage. In contrast, LCP learns $\small \widehat{F}^{(L)}_h(\cdot | \X = \x)$ on the calibration data $\mathcal{D}_n$, but the calibration step that consists of finding the adaptive $\alphat$ is computationally costly.

In this work, we propose to replace the Nadaraya-Watson (NW) estimator with the Quantile Regression Forest (QRF) algorithm \citep{meinshausen2006quantile} to estimate the distribution $\hV| \X = \x$ and use the LCP approach to calibrate the PI. The QRF algorithm is an adaptation of the Random Forest (RF) algorithm \citep{breiman1984classification}, which can be seen as an adaptive neighborhood procedure \citep{lin2006random}. It estimates the conditional c.d.f of $\hV | \X = \x$  as $\widehat{F}(r | \X = \x) = \sum_{i} w_n(\x, \X_i) \mathds{1}_{\hV_i \leq r}$ where the weights correspond to the average number of times where $\X_i$ falls in the same leaves of the RF as the observation $\x$.  Unlike kernel-based methods, the weights given by the RF depend on both feature input $\X_i$ and the residual $\hV_i$ due to the splits. We called this approach LCP-RF. This estimator has several advantages over the NW estimator. First, it is known to perform well in practice, even in high dimensions. It can handle both categorical and continuous variables. Additionally, it has interesting theoretical properties in high dimensions; under certain assumptions, it can be shown to be consistent and to adapt to the intrinsic dimension \citep{klusowski2021universal, scornet2015consistency}. To illustrate, we compute the PI of these methods using a random forest fitted on a toy model with input $\X \in [0, 7]^{21}$, and the target defined as $Y = \sin(X_1)^2 + 0.1 + 0.6 \times \epsilon \times \sin(2X_1)$, where $\epsilon \sim \mathcal{N}(0, 1)$, and $X_i \sim \mathcal{U}(0, 7)$ for all $i \in[21]$. As seen in Figure \ref{chap7:fig:toy_demo1}, the competitors LCP and SLCP fail to perform even on this very simple example with 1 active and 20 noise features, while our method benefits from the power of the Random Forest algorithm on tabular data \citep{grinsztajn2022tree}.

\begin{figure*}[ht!]
\centering
\subfigure{\includegraphics[width=0.3\textwidth]{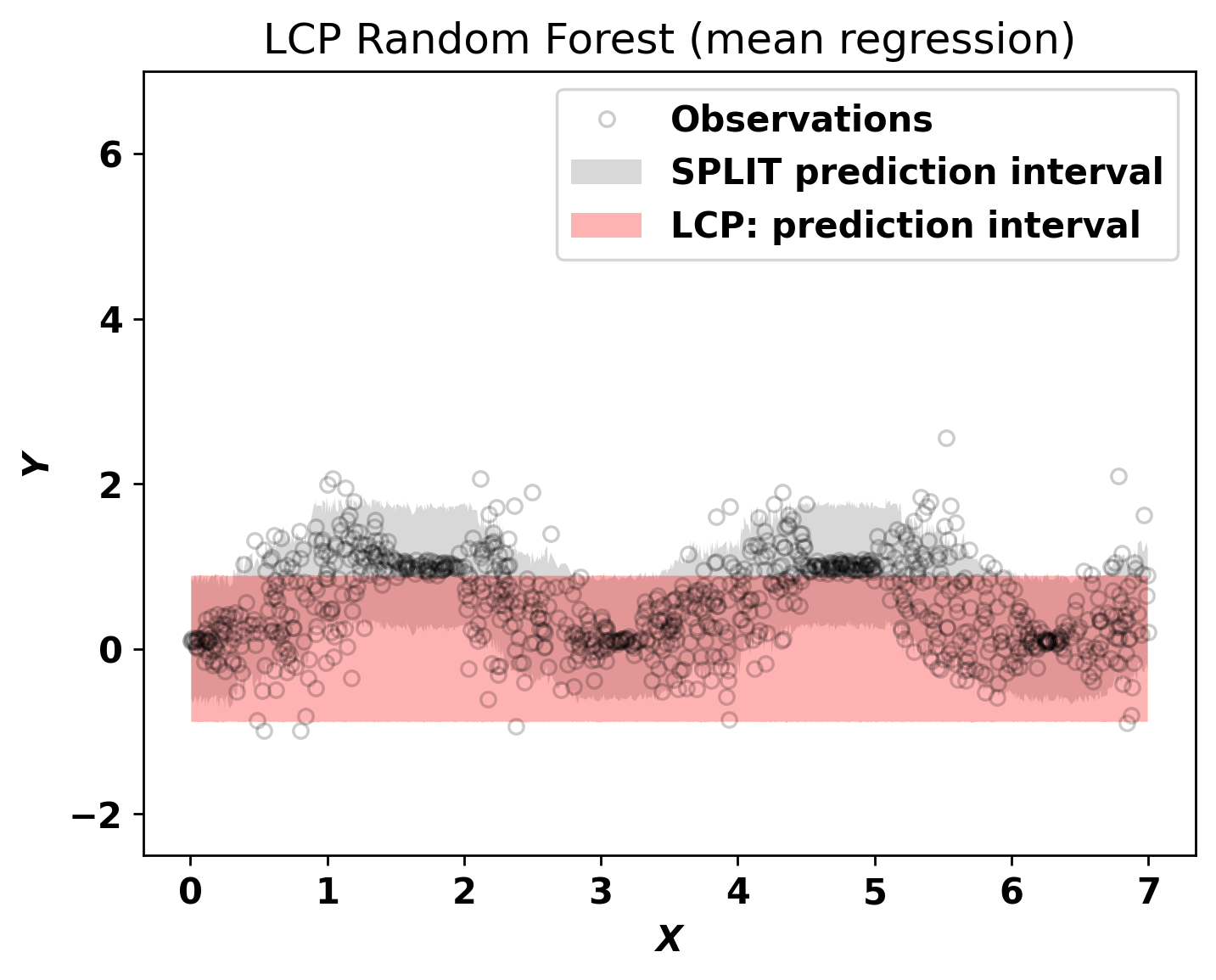}\label{chap7:fig:lcp}}
\subfigure{\includegraphics[width=0.3\textwidth]{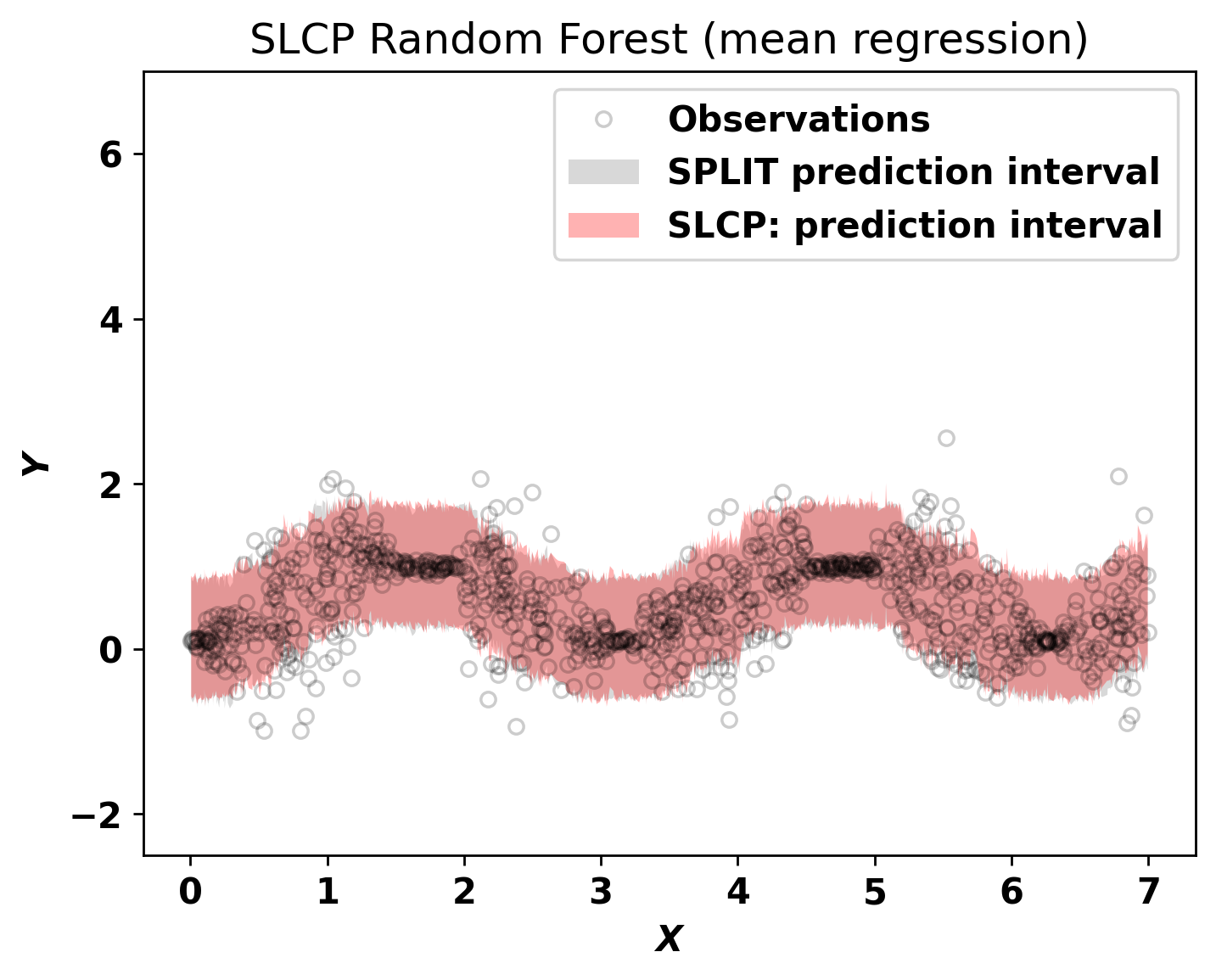}\label{chap7:fig:slcp}}
\subfigure{\includegraphics[width=0.3\textwidth]{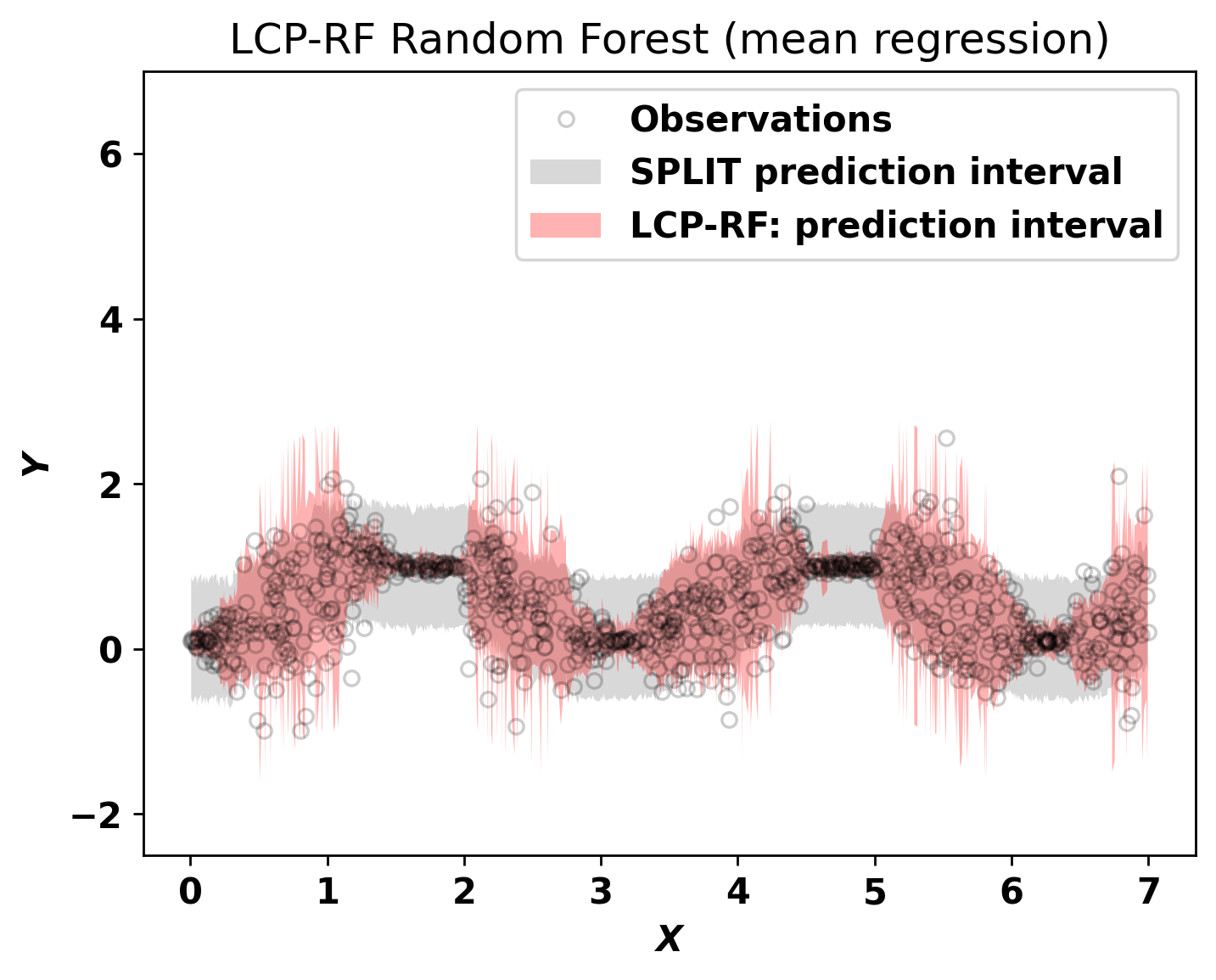}\label{chap7:fig:lcp_rf}}
\caption{Predictive interval at level $1-\alpha$ of SLCP, LCP and LCP-RF of a random Forest fitted on toy model $(\X, Y), \X \in [0, 7]^{21}$ and the target is defined as $Y = \sin(X_1)^2 + 0.1 + 0.6 \times \epsilon \times \sin(2X_1)$ with $\epsilon \sim \mathcal{N}(0, 1)$, and $X_i \sim \mathcal{U}(0, 7)$ for all $i \in[21].$}\label{chap7:fig:toy_demo1}
\end{figure*}
Additionally, we show that the learned weights of the RF can be used to create a relevant partition or groups/clusters of the input space. This allows for a more efficient computation of the LCP calibration and also allows for groupwise conformalization to give a more adaptive PI. In practice, it is often desirable to have a stronger coverage guarantee than marginal coverage. Consequently, we propose a further calibration step such that our PI satisfies training-conditional coverage. We also show that it achieves conditional coverage guarantee under suitable assumptions. 

An active area of research involves using a better nonconformity score to provide an adaptive prediction interval considering the variability of $\small Y|\X=\x$. Several methods have been proposed such as Conformal Quantile Regression (CQR) \citep{romano2019conformalized}, which uses score functions based on estimated quantiles, Locally Adaptive Split Conformal methods \citep{Lei2016DistributionFreePI} which use a scaled residual, and \cite{izbicki20a} proposed using the estimated conditional density as the conformity score. These methods incorporate different nonconformity scores $\small \hV(\cdot)$ that are better suited for handling the variability of $\small Y$. However, the extracted residuals $\hV_i$ of these nonconformity scores still depend on and vary according to $\small \X_{n+1}$. These methods are not competing with the LCP-RF approach as it can be applied to them to improve their PIs.


The main contributions of this paper are: (1) Developing an adaptive PI that better represents the uncertainty of a given model $\small \widehat{\mu}$ by using a QRF to learn the conditional distribution of the residuals $\small \hV(\X, Y)$, and utilizing the LCP framework to calibrate the resulting PI for marginal coverage, (2) Introducing a calibration step to achieve training-conditional coverage, (3) Exploiting the structure of the weights of the QRF to create groups for more adaptive PI and efficient computation through groupwise conformalization, (4) Showing that our methods achieve asymptotic conditional coverage under suitable conditions, (5) Demonstrating through simulations and real-world datasets that our methods outperform competitors LCP and SLCP, and providing a \href{https://github.com/salimamoukou/ACPI}{Python package} for the methods.

    



\section{Random Forest Localizer}

In this section, we present the Random Forest Localizer for constructing adaptive PI that depends on the test point $\small \Xtest$. The approach utilizes the learned weights of the RF and assigns higher weights to calibration samples that have residuals $\small \hV_i$ similar to $\small \hV_{n+1}$. This is based on the RF algorithm's ability to partition the input space by recursively splitting the data, resulting in similar observations with respect to the target variable (here the residuals) within each leaf node of the trees. The basic idea of the trees of the RF is to partition the input space into cells such that $\small Var(\hV(\X, Y) \;|\;\X \in A) \approx 0$ in each cell $\small A$. The corresponding weight of each calibration sample for $\small \Xtest$ is determined by the number of times it appears in the leaves of the trees where $\small \Xtest$ falls.

Random Forest (RF) is an ensemble learning method that utilizes the bagging principle \citep{breiman1996bagging} to combine k randomized trees derived from the CART algorithm \citep{breiman1984classification}. Each tree is constructed using a random sample of training data with replacement, and the best split at every node is identified by optimizing the CART-criterion among a random subset of variables. The predictions from all trees are then averaged to produce the final output of the forest. The Random Forest estimator can also be seen as an adaptive neighborhood procedure \citep{lin2006random}. Let assume we have trained the RF on $\small \mathcal{D}_n$, then for every instance $\boldsymbol{x}$, the observations in $\small \mathcal{D}_n$ are  weighted by $\small w_n(\boldsymbol{x}, \X_i)$, $i=1, \dots, n$. Therefore, the prediction of Random Forests and the weights can be rewritten as $\small m(\boldsymbol{x}, \Theta_{1},\dots, \Theta_{k}, \mathcal{D}_n) = \sum_{i=1}^{n} w_n(\boldsymbol{x}, \X_i) Y_{i}$ and 
\begin{align*}
   \small w_n(\boldsymbol{x}, \X_i) =  \sum_{l=1}^{k}  \frac{B_n(\X_i; \Theta_l) \; \mathds{1}_{\X_{i} \in  A_n(\boldsymbol{x}; \;\Theta_l)}}{k\times N_n(\boldsymbol{x};\; \Theta_l)}
\end{align*}where $\small \Theta_l, l=1, \dots, k$ are independent random vectors that represent the observations that are used to build each tree, i.e. the bootstrap samples, and the random subset of variables used in each node. $\small A_n(\boldsymbol{x}; \;\Theta_l)$ is the tree cell (leaf) containing $\small \boldsymbol{x}$, $\small N_n(\boldsymbol{x};\; \Theta_l)$ is the number of bootstrap elements that fall into $\small A_n(\boldsymbol{x}; \;\Theta_l)$, $\small B_n(\X_i; \;\Theta_l)$ is the number of times $\small \X_i$ has been chosen from the training data, and $\small w_n(\boldsymbol{x}, \X_i)$ represents the average number of times $\small \X_i$ appears in the same leaves as $\small \x$.

Random Forests can be used to estimate more complex quantities, such as cumulative hazard function \citep{ishwaran2008random}, treatment effect \citep{wager2017estimation}, and conditional density \citep{du2021wasserstein}. Quantile Regression Forests proposed by \cite{meinshausen2006quantile} use the same weights $\small w_n(\boldsymbol{x}, \X_i)$ as Random Forests to approximate the c.d.f $\small F(y |\X =\x)$ as:
\begin{equation}
     \small \widehat{F}(y |\boldsymbol{x}) = \sum_{i=1}^{n} w_n(\boldsymbol{x}, \X_i) \mathds{1}_{Y_i \leq y} \label{chap7:eq:estimator_boostrap_quantile}
\end{equation}
\textbf{Random Forest Localizer.} To approximate the estimated residuals $\hV | \X = \x$, we propose to fit a Quantile Regression Forest $\widehat{F}(\cdot|\x)$  on the calibration data $\widehat{\mathcal{D}}_n = \{(\X_i, \hV_i)\}_{i=1}^n$, and the estimator is defined as $\small \widehat{F}(r|\x) = \sum_{i=1}^{n + 1} w_n(\boldsymbol{x}, \X_i) \mathds{1}_{\hV_{i} \leq r}$, where $\hV_{n+1}=+\infty$ unless specified. It's worth noting that this estimator is slightly different from (\ref{chap7:eq:estimator_boostrap_quantile}), as it includes the observation $\Xtest$ in the weighted sum. We will see later that this addition would be essential to prove the marginal coverage property of our method. Using this estimator, a natural PI for $\hV_{n+1}$ is:
\begin{align} \label{chap7:eq:ci_v0}
   \small C_V(\X_{n+1}) = \left\{v: v \leq \mathcal{Q}\left(1 - \alpha;\; \widehat{F}(\cdot|\Xtest)\right) \right\}
\end{align} 
Recall that we obtain the prediction interval $\small C(\Xtest)$ for $\small Y_{n+1}$ by inverting the nonconformity score $\hV(\Xtest, \cdot)$ using $C_V(\X_{n+1})$ as in equation (\ref{chap7:eq:predictive_set}). Thus, the real quantity of interest is $\small C_V(\Xtest)$. The question at hand is whether the PI $\small C_V(\Xtest)$ defined in (\ref{chap7:eq:ci_v0}) satisfies marginal coverage. If $\small w_n(\Xtest, \X_i) = \frac{1}{n+1}$, we have $\small \mathcal{Q}\left(1 - \alpha;\; \widehat{F}(\cdot|\Xtest) \right) = \hV_{(\lceil (1-\alpha)(n + 1) \rceil)}$ and thanks to the quantile lemma  \citep{tibshirani2019conformal} and exchangeability of the $\hV_i$, we have the marginal coverage. However, If $\small \widehat{F}(\cdot | \Xtest)$ gives non-equal weights to the calibration samples, it is no longer the case. Recent methods have been proposed by \cite{tibshirani2019conformal} and \cite{barber2022conformal} that achieve marginal coverage when using reweighting. However, these methods cannot be applied to calibrate our PI, as they work under different assumptions. The method introduced by \cite{barber2022conformal} assumes that the weights do not depend on the data, while the method proposed by \cite{tibshirani2019conformal} handles data-dependent weights but assumes a covariate shift, where the training and test data have different input distributions but the same conditional distribution $P_{Y | \X}$.

To calibrate our PI, we use the Localized Conformal Prediction (LCP) framework \citep{guanlocalizer} to select an appropriate level $\alphat$ of the quantile used in the PI (\ref{chap7:eq:ci_v0}) to ensure marginal coverage at level $1-\alpha$. Hence,  the PI becomes
\begin{align} \label{chap7:eq:pi_a}
    \small C_V(X_{n+1}) = \left\{v: v \leq \mathcal{Q}\left(\alphat;\; \widehat{F}(\cdot|\Xtest)\right) \right\}.
\end{align}


\section{Weighted Conformal Prediction}
In this section, we give a comprehensive overview of the LCP framework of \cite{guanlocalizer} with the Random Forest Localizer for completeness. Additionally, we describe our calibration approach that guarantees training-conditional coverage, and how we leverage the weights of the RF to improve the LCP calibration process and produce more adaptive prediction intervals. For ease of reading, we follow \cite{guanlocalizer} and introduce $\small \mathcal{F}_i = \widehat{F}(\cdot | \X_i) = \sum_{j=1}^{n} w_n(\X_i, \X_j) \mathds{1}_{\hV_j \leq \cdot} + w_n(\X_i, \X_{n+1}) \mathds{1}_{\hV_{n+1} \leq \cdot}$
as the estimated c.d.f of $\small \hV$ given $\small \X_i$ given by the Quantile Regression Forest. As $\hV_{n+1}$ is not observed and we need to consider the possible values of $\small \hV_{n+1}$  for  constructing the PI, we introduce the additional notations  $\small \mathcal{F}_i^v$ for the c.d.f when $\small \hV_{n+1} = v$ if $v$ is finite, and  $\small \mathcal{F} = \mathcal{F}_{n+1}^{\infty}$ if $\small \hV_{n+1} = +\infty$. 

\subsection{Localized Conformal Prediction \citep{guanlocalizer}}
The following lemma is the cornerstone of the LCP framework. It shows how to achieve marginal coverage by properly selecting the level $\alphat$ of the quantile of the localizer.
\begin{lemma} \label{chap7:theorem:lcp}
    Let $\alphat$ be the smallest value in $\small \Gamma = \left\{ \sum_{j=1}^{k} w_n(\X_i, \X_j): i=1, \dots, n+1;\; k=1, \dots, n+1\right\}$ such that \begin{equation*}
        \small  \sum_{i=1}^{n+1}\frac{1}{n+1} \mathds{1}_{\hV_i \leq \mathcal{Q}(\alphat;\;  \mathcal{F}_i )} \geq 1 - \alpha,
    \end{equation*}
then $\small P^{n+1}\left(\hV_{n+1} \leq \mathcal{Q}(\alphat; \; \mathcal{F}_{n+1})\right) \geq 1-\alpha$, or equivalently $\small P^{n+1}\left(\hV_{n+1} \leq \mathcal{Q}(\alphat; \; \mathcal{F})\right) \geq 1-\alpha$.
\end{lemma}
Remember that both $\small \alphat$ and $\small \mathcal{F}_{n+1}$ depends on $\small \widehat{\mathcal{D}}_{n} = \{(\X_i, \hV_{i})\}_{i=1}^{n}$ and $\small (\Xtest, \hV_{n+1})$, but we won't specify them for clarity. 
Now, we can use Lemma \ref{chap7:theorem:lcp} to test $\small H_0: \hV_{n+1} = v$ for each $\small v \in \overline{\mathbb{R}}$ under exchangeability, then invert the test to construct the PI. $\small C_V(\Xtest)$ consists of all values $v$ that are not rejected by this test.  The resulting PI has marginal coverage as shown in the following theorem.
\begin{theorem} \label{chap7:lemma:lcp}
    At $\small \hV_{n+1} = v$, let define $\small \alphat$ that depends on $\small \widehat{\mathcal{D}}_n$ and $\small (\Xtest, v)$ to be the smallest value $\small \alphat \in \Gamma$ such that 
    \begin{equation*}
        \small \sum_{i=1}^{n + 1} \frac{1}{n+1} \mathds{1}_{\hV_i \leq \mathcal{Q}(\alphat; \; \mathcal{F}^v_i) } \geq 1 - \alpha.
    \end{equation*}
    Set $\small C_V(V_{n+1}) = \left\{ v: v \leq \mathcal{Q}\left(\alphat; \; \mathcal{F}\right)\right\}$ and $\small C(Y_{n+1}) = \left\{ y: \hV(\Xtest, y) \leq \mathcal{Q}\left(\alphat; \; \mathcal{F}\right)\right\}$, then by construction, Lemma \ref{chap7:theorem:lcp} gives $\small P^{n+1}\left(Y_{n+1} \in C(\Xtest)\right) = P^{n+1}\left(\hV_{n+1} \in C_V(\Xtest)\right) = P^{n+1}\left(\hV_{n+1} \leq \mathcal{Q}\left(\alphat; \; \mathcal{F}\right)\right) \geq 1 - \alpha$.
\end{theorem}
At this point, the LCP  method is not practical as it requires computing $\small \alphat$ for every possible value of $\small v \in \overline{\mathbb{R}}$ in order to construct the prediction interval. This process can be extremely time-consuming and computationally intensive. However, \cite{guanlocalizer} shows that the computation of $\small C_V(\Xtest)$ can be done efficiently thanks to its interesting properties. Specifically, if $\small v$ is accepted in $\small C_V(\Xtest)$, all $\small v^\prime \leq v$ are also accepted. Hence, it is sufficient to find the largest accepted value $v^\star$. Additionally, as $\small \mathcal{Q}(\alphat; \;\mathcal{F}_i^v)$ is non-decreasing in both $\small \alphat$ and $\small v$, and piece-wise constant in $\small \alphat$, with value changes only occurring at different $\small \hV_i$, it can be proven that the largest value is attained by one of the estimated residuals $\small \hV_{k^\star}$ with $\small k^\star \in [n]$. Therefore, the closure $\small \bar{C}_V(\Xtest)$ of $\small C_V(\Xtest)$ is given by $\small \bar{C}_V(\Xtest)= \left\{v: v \leq \hV_{k^\star} \right\}$ for some $\small k^\star \in [n]$. The following Lemma shows how to find $V_{k^\star}$.



\begin{lemma} \label{chap7:eq:lemma_sk}
    We denote $\small \hV_{(1)}, \dots, \hV_{(n)}$ the order statistics of the nonconformity score of the calibration samples, set $\small \hV_{(n+1)} = +\infty$, and $\small \tilde{\theta}_k = \sum_{i=1}^{n} w_n(\Xtest, \X_{i}) \mathds{1}_{\hV_{i} < \hV_{(k)}}$. Let  $\small k^\star \in \{1, \dots, n+1\}$ the largest index such that \begin{equation} \label{chap7:eq:lemma_lcp}
        \small S(k) := \sum_{i=1}^{n}\frac{1}{n+1} \mathds{1}_{\hV_i \leq \mathcal{Q}\left(\tilde{\theta}_{k^\star};\; \mathcal{F}_i^{\hV_{(k^\star)}}\right)} < \alpha.
    \end{equation}
Then,  $ \bar{C}_V(\Xtest)= \left\{v: v \leq \hV_{(k^\star)} \right\}$ is the closure of $\small C_V(\Xtest)$.
\end{lemma} \cite{guanlocalizer} also proposed an algorithm that computed $\small S(k)$ in $\small \mathcal{O}(n \log(n))$ time. The description of the algorithm can be found in the original paper.

\subsection{Training-Conditional coverage for LCP-RF}

 Here, we consider training-conditional coverage or PAC predictive interval guarantees for the LCP-RF. Let's consider the coverage rate given a calibration set $\mathcal{D}_n$ as $\small \cov(\mathcal{D}_n) = P(\hV_{n+1} \in C_V(\Xtest) \; \vert \; \mathcal{D}_n)$ where the probability is taken with respect to the test point $(\Xtest, \hV_{n+1})$. The PAC predictive interval ensures that for most draws of the calibration samples $\mathcal{D}_n \sim P^n$, we have $\cov(\mathcal{D}_n) \geq 1 - \alpha$. Formally, $\exists \delta$ s.t. $P^n\left(\cov(\mathcal{D}_n) \geq 1 - \alpha \right) \geq 1 - \delta$.

We use a two-step approach to ensure training-conditional coverage for the LCP-RF. First, we use a portion of the calibration samples to ensure marginal coverage by applying the LCP approach. Next, we use a separate portion of the calibration samples to learn a correction term, which is then added to the LCP-RF approach to ensure training-conditional coverage. This approach is similar to the one used in \citep{kivaranovic2020adaptive}. We split the calibration set $\small \widehat{\mathcal{D}}_n$ into two sets $\small \widehat{\mathcal{D}}^i_{n_i} = \left\{(\X^i_1, \hV^i_1), \dots, (\X^i_{n_i}, \hV^i_{n_i}) \right\}$ for $\small i=1, 2$  with $\small n_1 + n_2=n$. We train the Quantile Regression Forest on $\small \widehat{\mathcal{D}}^1_{n_1}$, and compute PI for the observations in the second set $\small \widehat{\mathcal{D}}^2_{n_2}$ using the LCP-RF. The PI of each $\small i \in \widehat{\mathcal{D}}^2_{n_2}$ is $\small C_V(\X^2_i) = \left\{v: v \leq \mathcal{Q}(\alphat(\X^2_i);\; \mathcal{F}^{2, \infty}_i) \right\}$, where $\small \alphat(\X^2_i)$ is the adapted level $\small \alphat$ to have marginal coverage if $\small \X^2_i$ is the test point and $\small \mathcal{F}^{2, \infty}_i = \sum_{j=1}^{n_1 + 1} w_n(\X^2_i, \X^1_j) \mathds{1}_{\hV^1_j \leq V}$ is the estimated residual distribution learn on $\small \mathcal{D}^1_{n_1}$ evalued on $\small \X^2_i$ where we set $\small \X^1_{n_1 + 1} = \X^2_i$ and $\small \hV^1_{n_1 + 1} = +\infty$. The following lemma shows how we can correct the corresponding $\small \alphat(\X^2_i)$ by adding a correction term $\small \textcolor{red}{\widehat{\alpha}}$ to ensure PAC coverage.
\begin{theorem}\label{chap7:theo:pac_interval}
    Let $ \epsilon>0$,  $\alpha - \epsilon >0$ and $ \textcolor{red}{\widehat{\alpha}}$ be the smalest value in $T = \{\alpha_0=0,\dots, \alpha_K=\alpha\}$  s.t.
    \begin{align}
         \sum_{i=1}^{n_2} \frac{1}{n_2} \mathds{1}_{\hV^{2}_{i} \leq  \mathcal{Q}\left( \alphat(\X^{2}_{i}) + \textcolor{red}{\widehat{\alpha}};\; \mathcal{F}^{2, \infty}_{i}\right)} \geq 1 - \alpha.
    \end{align}
    Then, we have  $ P^{n_1}\left\{ \cov(\mathcal{D}_{n_1}) \geq 1 - \alpha - \epsilon \right\} \geq 1 - \delta$ with $\delta=K \exp(-2n_2\epsilon^2)$ and $ \cov(\mathcal{D}_{n_1}) = P \left\{ \hV_{n+1} \leq \mathcal{Q}\left(\alphat(\X_{n + 1}) + \textcolor{red}{\widehat{\alpha}};\; \mathcal{F}^{\infty}_{n+1}\right)\;\big|\; D_{n_1}\right\}$.
\end{theorem}
\textbf{Remark.} This result is valid under the i.i.d assumption and not under exchangeability as the other results of the paper. We suggest choosing a grid $T \subset [0, \alpha]$ as we have observed in most practical scenarios that $\alphat(\X_{n + 1}) \approx 1 - \alpha$.  However, the central idea remains unaltered - to select a grid that enables transitioning from $\alphat(\X_{n+1})$ to 1. Additionally, as $\alphat(\X_{n + 1}) + \textcolor{red}{\widehat{\alpha}}$ may be above $1$, we define $\alphat(\X_{n + 1}) + \textcolor{red}{\widehat{\alpha}} := (\alphat(\X_{n + 1}) + \textcolor{red}{\widehat{\alpha}}) \lor 1$. 
\subsection{Clustering using the weights of LCP-RF}
In this section, we analyze the weights of the Random Forest Localizer and show that it offers several benefits compared to traditional kernel-based localizer. These benefits include faster computation and more adaptive PIs. One key difference between the RF localizer and kernel-based localizer is that the RF localizer's weights are sparse, i.e., many weights being zero. For a given test point $\small \Xtest$, if $\small w_n(\Xtest, \X_i) = 0$, then the estimated c.d.f $\small \mathcal{F}_i $ does not depend on the value of $\small \hV_{n+1}$. Thus, it may not be necessary to use  $\mathcal{F}_i$ in the LCP's marginal calibration (eq. \ref{chap7:eq:lemma_lcp} in Lemma \ref{chap7:eq:lemma_sk}).

The weights defined by the Random Forest Localizer have a structure that can be utilized to group similar observations together before applying the calibration steps. Indeed, we can view the weights of the RF on the calibration set as a transition matrix or a weighted adjacency matrix $\small G$ where $\small G_{i, j} = w_n(\X_i, \X_j)$ and $\small \forall j \in [n],$ we have $\sum_{i=1}^{n} w_n(\X_j, \X_i) = \sum_{i=1}^{n} w_n(\X_i, \X_j) = 1.$

To exploit this structure, we propose to group observations that are connected to each other and separate observations that are not connected. This can be done by considering the connected components of the graph represented by the matrix $G$. Assume that $G$ has $L$ connected components represented by the disjoint sets of vertices $\small G_1, \dots, G_L$, defined such that for any $\small \X_i, \X_j \in G_l$, there is a path from $\small \X_i$ to $\small \X_j$, and they are connected to no other vertices outside the vertices in $\small G_l$. This leads to the existence of a partition of the input space $\small \cup_{i=1}^L R_i = \mathcal{X}$, where $\small \forall k, l \in [L], R_l \cap R_q = \varnothing$, and for all $\small \X_i \in R_p, X_j \in R_q$, we have $\small w_n(\X_i, \X_j) = 0$. The regions $\small R_i$ is defined as $\small R_i = \big\{\x\in \mathbb{R}^d: \exists \X \in G_i, w_n(\x, \X) > 0 \text{ and } \forall \X^\prime \in G_k, k\neq l, \; w_n(\x, \X^\prime)=0 \big\}$.
By definition of the weights, we can also define $R_i$ using the leaves of the RF as $\small R_i = \bigcup_{\X_i \in G_i}\Big[  \cup_{l=1}^{k} A_n(\X_i, \Theta_l)\Big]$.
This shows that the $R_i$ are connected space. Hence, we can apply the conformalization steps separately on each group and use only the observations that are connected to the test point. By using the conformalization by group, we reduce the computation of $S(k)$ in Lemma \ref{chap7:eq:lemma_sk} needed for the computation of the PI from $\small \mathcal{O}(n \log(n))$ to $\small \mathcal{O}\left(\big|\textbf{R}(\Xtest)\big| \log\big|\textbf{R}(\Xtest)\big|\right)$, since we only use the observations in the region where $\Xtest$ belongs in the calibration step. This results in a more accurate and efficient PI. In addition, no coverage guarantees are lost as the $R_i$ forms a partition. We prove the marginal coverage of the group-wise LCP-RF in the appendix.

In some cases, the graph may have a single connected component. Consequently, we propose to regroup calibration observations by (non-overlapping) communities using the weights of the RF. This involves grouping the nodes (calibration samples) of the graph into communities such that nodes within the same community are strongly connected to each other and weakly connected to nodes in other groups. Various methods exist for detecting communities in graphs, such as hierarchical clustering, spectral clustering, random walk, label propagation, and modularity maximization. A comprehensive overview of these methods can be found in  \citep{Schaeffer2007SurveyGC}. Nonetheless, it is challenging to determine the most suitable approach as the selection depends on the particular problem and characteristics of the graph. In our experiments, we found that the popular Louvain-Leiden \citep{traag2019louvain} method coupled with Markov Stability \citep{delvenne2010stability} is effective in detecting communities of the learned weights of the Random Forest. However, any clustering method can be used depending on the specific application and dataset.

Let's assume a graph-clustering algorithm $C$ that returns L disjoint clusters $C(\mathcal{D}_n) = \{ C_1, \dots, C_L \}$. Note that contrary to connected components, we can have $\X \in C_i, \; \X^\prime \in C_j$ and $w_n(\X, \X^\prime) \neq 0$, therefore it's more difficult to define the associated regions $R_1, \dots R_L$ that form a partition of $\mathcal{X}$ s.t. for any $\X \in C_i$, then $\X \in R_i$. We define $R_i$ as the set of points $\x$ that assigns the highest weights to the observations in cluster $C_i$. As $w$ can be interpreted as a transition matrix, we define $R_i$ as the set of $\X$ such that $ \widehat{p}(\X \in C_i) > \widehat{p}(\X \in C_k)$ for all $k\neq i$, where the probability is computed using the weights of the forest. Formally, $R_i$ can be represented as 
\begin{equation*}
    \small R_i = \left\{ \x \in \mathbb{R}^d: \sum_{j \in C_i} w_n(\x, \X_j) > \sum_{j \in C_k} w_n(\x, \X_j),\; k\neq i\right\}.
\end{equation*}
However, we also need to define another set for observations that are "undecidable", i.e., belong to several groups at the same time. We define this set as $\small \bar{R} =\left\{ \x \in \mathbb{R}^d: \exists k, l \in [L], \sum_{j \in C_l} w_n(\x, \X_j) = \sum_{j \in C_k} w_n(\x, \X_j)\right\}$. Finally, we get marginal/PAC coverage as we do with the connected components case by applying the calibration step conditionally on the groups $\small R_1, \dots, R_L$ and $\bar{R}$.

\section{Asymptotic conditional coverage}
Here, we study the conditional coverage of LCP-RF. It is widely recognized that obtaining meaningful distribution-free conditional coverage is impossible with a finite sample \citep{Lei2014DistributionfreePB, vovk2012conditional}. Below, we demonstrate the asymptotic conditional coverage of LCP-RF while making weaker assumptions than the original LCP based on kernel \citep{guanlocalizer}.

\begin{assumption} \label{chap7:prop:assym1}
$\small \forall r \in \mathbb{R}$, the c.d.f $\x \rightarrow \small F(r | \X=\x)$ is continuous.
\end{assumption}
Assumption \ref{chap7:prop:assym1} is necessary to get uniform convergence of the RF estimator.
\begin{assumption} \label{chap7:prop:assym2}
For $l \in [k]$, the variation of the conditional cumulative distribution function within any cell goes to $0$, i.e., $\small \forall \x \in \mathbb{R}^d, \forall r \in \mathbb{R}, \sup_{\boldsymbol{z} \in A_n(\boldsymbol{x}; \;\Theta_l)} |F(r | \boldsymbol{\boldsymbol{z}}) - F(r|\boldsymbol{x})| \overset{a.s}{\to} 0$.
\end{assumption}
Assumption \ref{chap7:prop:assym2} allows for control of the approximation error of the RF estimator. \cite{scornet2015consistency} show that this is true when the data come from additive regression models \citep{additiveStone}, and \cite{elie2020random} show that it holds for a more general class, such as product functions or sums of product functions. This result also applies to all regression functions, with a slightly modified version of RF, where there are at least a fraction $\gamma$ observations in child nodes and the number of splitting candidate variables is set to 1 at each node with a small probability. Therefore, we do not need to assume that for all $r$,  $F(r|.)$ is Lipschitz, as required in LCP \citep{guanlocalizer}, which is a much stronger assumption.

\begin{assumption} \label{chap7:prop:assym3}
Let $\small k$ and the number of bootstrap observations in a leaf node $\small N_n(\boldsymbol{x};\; \Theta_l)$, s.t. there exists $\small k = \mathcal{O}(n^\alpha)$, with $\small \alpha > 0$, and $\small \forall \x \in \mathbb{R}^d$, $ N_n(\boldsymbol{x};\; \Theta_l) = \Omega\footnote{$\footnotesize f(n) = \Omega(g(n)) \iff \exists k > 0, \exists n_0 >0, \forall n \geq n_0, |f(n)| \geq |g(n)|.$ }(\sqrt{n} (ln (n))^\beta)$, with $\small \beta > 1$ a.s. 
\end{assumption}
Assumption \ref{chap7:prop:assym3} allows us to control the estimation error and means that the cells should contain a sufficiently large number of points so that averaging among the observations is effective. It can be enforced by adjusting the hyperparameters of the RF.

Under these assumptions, we prove that the selected $\small \alphat$ when $\small \hV_{n+1} = v$ given by the LCP-RF converges to $\small 1-\alpha$,  and the resulting PI achieves the target level $\small 1-\alpha$.
\begin{theorem} \label{chap7:theo:lcp_cond}
    Let $\small \alphat$ and $\small C_V(\Xtest)$ define as in Theorem \ref{chap7:lemma:lcp}. Under assumptions \ref{chap7:prop:assym1}-\ref{chap7:prop:assym3}, we have $\small \lim_{n \rightarrow \infty} P\left( \hV_{n+1} \in C_V(\Xtest) \; |\;  \Xtest\right) = 1-\alpha$ and for all $\small \epsilon > 0$, $\small \lim_{n \rightarrow \infty} P\left( \max_v|\alphat - (1-\alpha)| < \epsilon\; |\;  \Xtest\right) = 1$ a.s.
\end{theorem}

\section{Experiments}
We evaluate the performance of our proposed methods: LCP-RF (Random Forest Localizer with marginal and training-conditional calibration), LCP-RF-G (LCP-RF with groupwise calibration) and QRF-TC (Random Forest Localizer with only training-conditional calibration) against their competitors SPLIT (split-CP), SLCP and LCP. We used the original implementation of SLCP and LCP and tuned the kernel widths as described in their respective papers. We test the methods on simulated data with heterogeneous output and 3 real-world datasets from UCI \citep{uci_dataset}: bike sharing demand (bike, n=10886, p=12), california house price (cali, n=20640, p=8), and community crime (commu, n=1993, p=128). The datasets are divided into three sets, namely the training set $(40\%)$, calibration set $(40\%)$, and the test set $(20\%)$. To ensure that the model's error is not constant across all observations, we created a hole in the data by removing all observations from the training set whose output exceeds the 0.7-quantile of the training outputs.  The PI is computed on the test sets at a level of $1-\alpha=0.9$.

We consider two nonconformity scores: mean score $\small \hV(\X, Y) = |Y - \widehat{\mu}(\X)|$ where $ \widehat{\mu}$ is mean estimate, and quantile score $\small \hV^Q(\X, Y)=\max\left\{\widehat{q}_{\alpha/2}(\X)-Y, Y-\widehat{q}_{1-\alpha/2}(\X)\right\}$ where $\small \{\widehat{q}_{\alpha/2}, \widehat{q}_{1-\alpha/2} \}$ are quantile estimates at level $\small \alpha/2$ and $\small 1 - \alpha/2$ respectively. We use XGBoost \citep{chen2016xgboost} of scikit-learn \citep{pedregosa2011scikit} with default parameters as the mean estimate $\small \widehat{\mu}$ in our experiments. We leave the analysis of different models and the quantile score for the appendix. We denote $\small C^{m}(\Xtest) = \big[\widehat{\mu}(\Xtest) \pm q^{m}(\Xtest)\big]$ the PI of each method $m$, and the oracle PI as $\small C^{\star}(\Xtest) = [\widehat{\mu}(\Xtest) \pm q^\star(\Xtest)]$ where $\small q^\star(\Xtest) = \mathcal{Q}\left(1-\alpha; \; F_{\hV_{n+1}| \Xtest}\right)$.
\begin{figure*}[ht!]
\centering
\subfigure[sim (lengths)]{\includegraphics[width=0.197\textwidth]{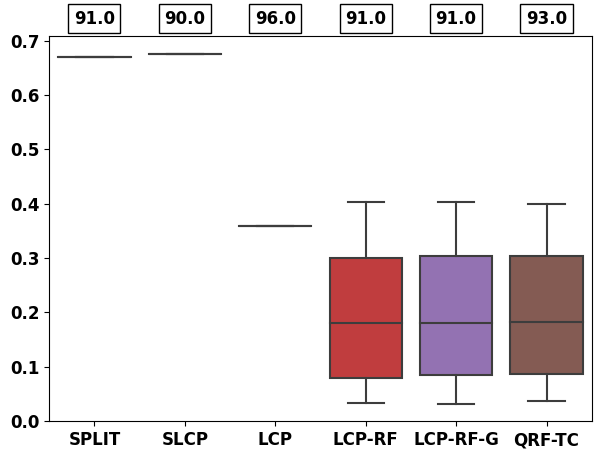}\label{chap7:fig:lengths_sim}}
\subfigure[bike (lengths)]{\includegraphics[width=0.21\textwidth]{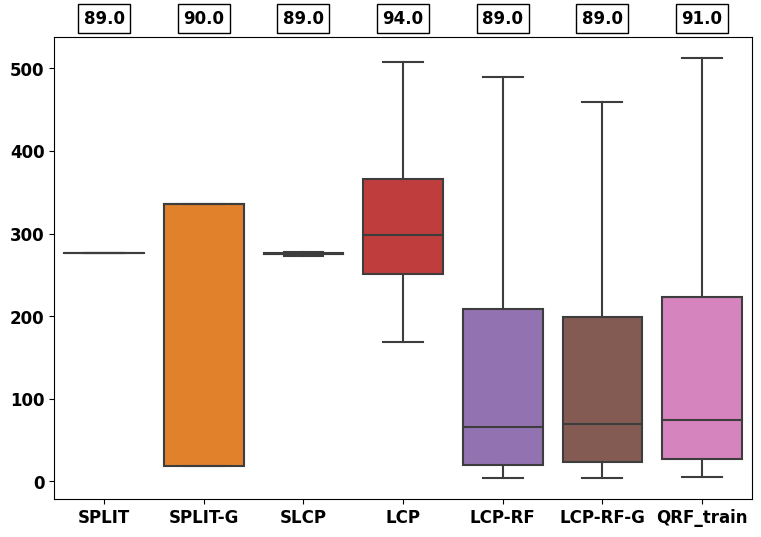} \label{chap7:fig:bio_lengths}}
\subfigure[cali (lengths)]{\includegraphics[width=0.21\textwidth]{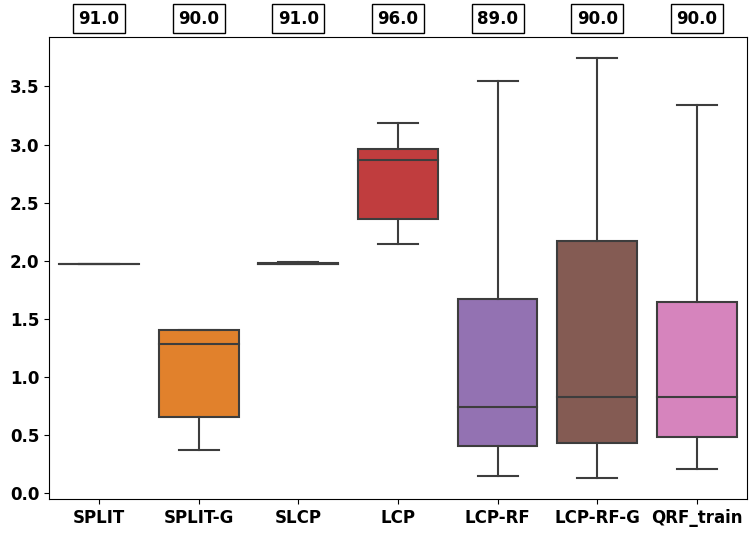}}
\subfigure[commu (lengths)]{\includegraphics[width=0.21\textwidth]{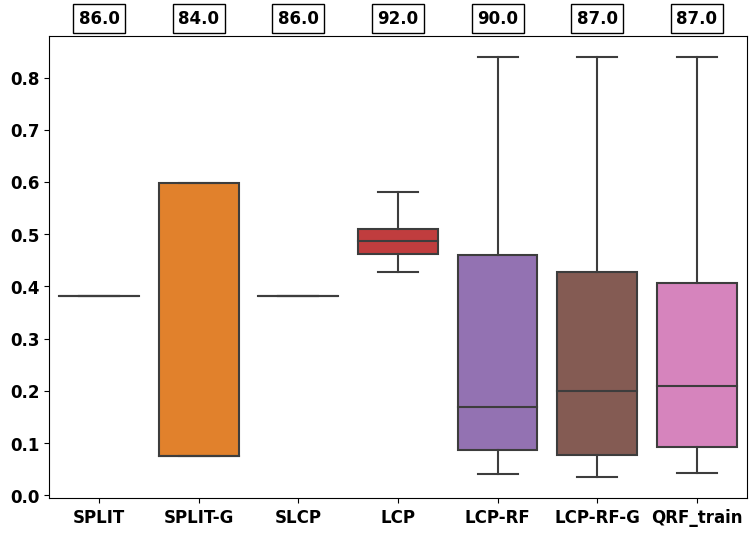}\label{chap7:fig:star_lengths}}
\subfigure[sim ($ err_{n+1}$)]{\includegraphics[width=0.197\textwidth]{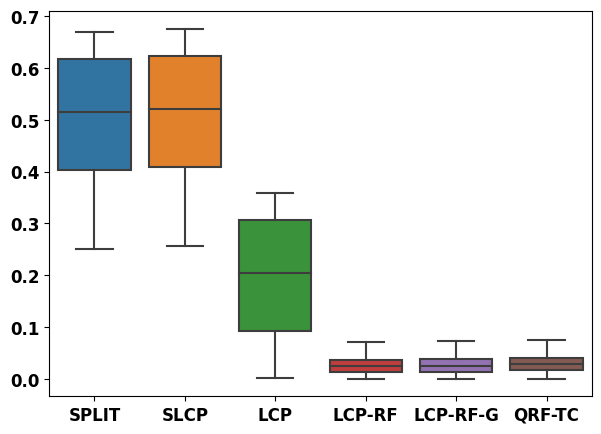}\label{chap7:fig:residuals_sim}}
\subfigure[bike ($\widehat{err}_{n+1}$)]{\includegraphics[width=0.21\textwidth]{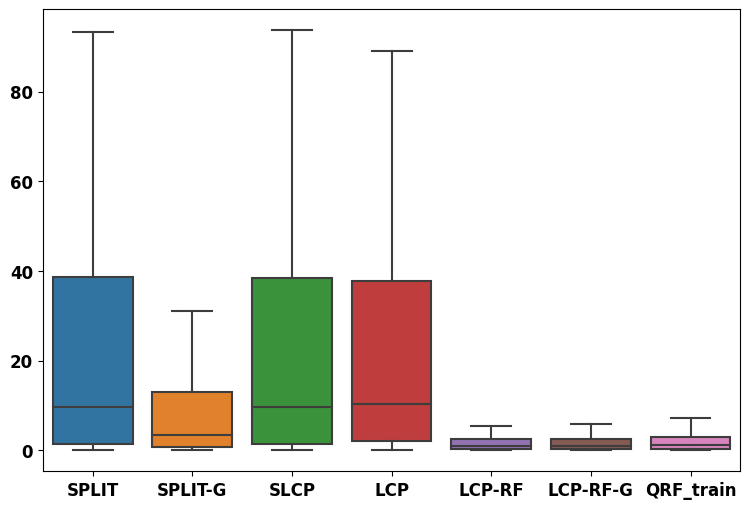}\label{chap7:fig:bio_residuals}}
\subfigure[cali ($\widehat{err}_{n+1}$)]{\includegraphics[width=0.21\textwidth]{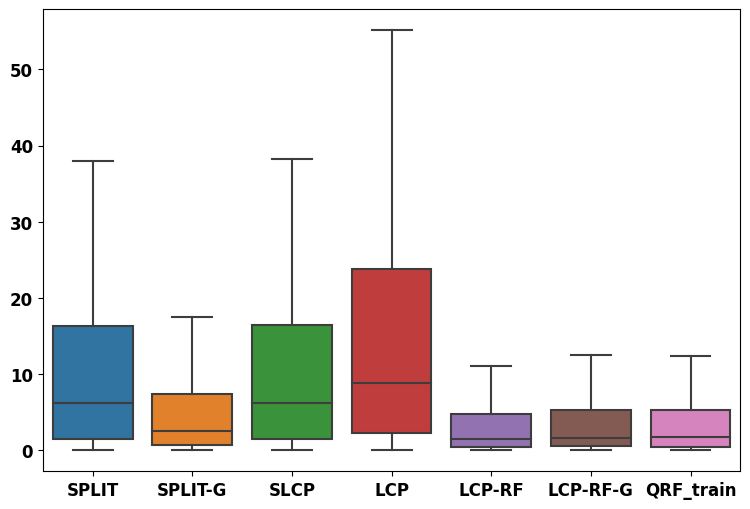}}
\subfigure[commu ($\widehat{err}_{n+1}$)]{\includegraphics[width=0.21\textwidth]{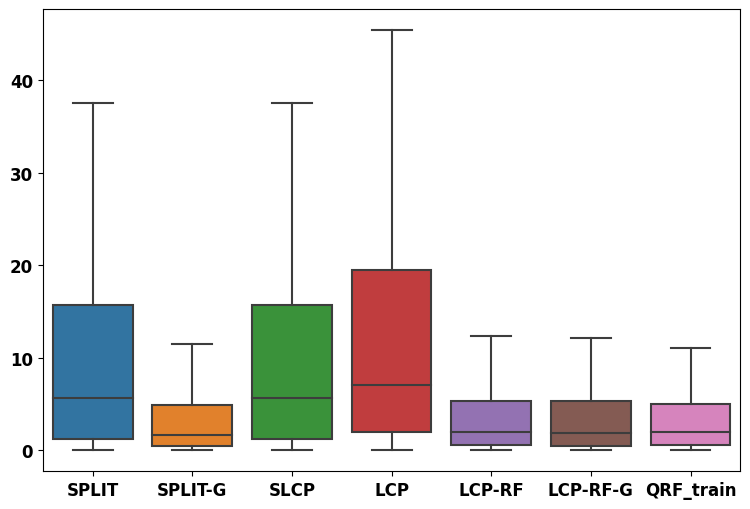}\label{chap7:fig:star_residuals}} 
\caption{PI lengths and errors of the different methods with the mean score. The training-conditional coverages are at the top of the figure.}
\label{chap7:fig:realdata_results}
\end{figure*}

The simulated data (sim) is defined as: $\small \X \in [0, 1]^{50}$, $\X_i \sim \mathcal{U}([0, 1])$ for all $i \in [50]$ and $\small Y = \X_1 + \epsilon \times \X_1/(1 +\X_1)$ where $\small \epsilon \sim \mathcal{N}(0, 1)$. In figure \ref{chap7:fig:residuals_sim}, we compute the absolute relative distance between the PI of each method and the oracle PI as $\small err_{n+1} = |q^m(\Xtest) - q^\star(\Xtest)|/q^\star(\Xtest)$ showing that our methods are much closer to the oracle PI than its competitors. SLCP and SPLIT are close, but they are less accurate than LCP. Figure \ref{chap7:fig:lengths_sim} shows that most methods provide training-conditional coverage or empirical coverage over the test points at nearly $90\%$. Our methods give varied intervals while the others have almost constant intervals.

The analysis of real-world data is more challenging because we don't have the oracle PI. To evaluate the effectiveness of the methods, we compare the length of the PI $\small q^m(\Xtest)$ to the true error of the model $\small \hV_{n+1} = \hV(\Xtest, Y_{n+1})$.  Indeed, a larger error of the model should result in a larger PI. Note that if $\small Y_{n+1}|\Xtest$ does not vary too much then $\small \hV_{n+1} \approx q^\star(\Xtest)$. We denote $\small \widehat{err}_{n+1} = |q^m(\Xtest) - \hV_{n+1}|/\hV_{n+1}$ as the model's fidelity errors. We also introduce a new method (SPLIT-G), corresponding to groupwise split-CP using the groups defined by the RF's weights.

Figure \ref{chap7:fig:realdata_results} summarizes the results on the 3 real-world datasets. Starting with average coverage (top of the figure), most methods have empirical coverage at nearly exact nominal levels for all datasets. Our methods are slightly lower, which could be explained by the sample splitting used for the PAC interval calibration. Indeed, the bound in Theorem \ref{chap7:theo:pac_interval} depends on the size of the data and as we split the calibration set in two, we lose a bit in statistical efficiency.

The top figures (\ref{chap7:fig:bio_lengths}-\ref{chap7:fig:star_lengths}) display the distribution of the lengths of the PI, while the four figures at the bottom (\ref{chap7:fig:bio_residuals}-\ref{chap7:fig:star_residuals}) show the distribution of fidelity errors of the model $\widehat{err}_{n+1}$. Overall, our methods significantly outperform the others in terms of the model's uncertainty fidelity and adaptiveness of lengths. SLCP fails to provide any significant improvement over the standard split-CP. This could be due to the fact that it learns the localizer on the residuals of the training set, which may not represent the residuals of the calibration data, thereby leading to overfitting. While LCP-RF-G and QRF-TC are faster than LCP-RF, their performance are similar. In these datasets, we suspect that the RF localizer is so accurate that it is difficult to distinguish between the groupwise LCP-RF and the LCP-RF. However, we observe that by using the groups defined by the RF with the split-CP (SPLIT-G), we were able to improve the PI of split-CP. This demonstrates how our methods can improve the performance of any CP approach by just utilizing the weights established by the RF.

\section{Conclusion}
In this work, we have significantly enhanced the applicability of the Localized Conformal Prediction framework, which previously only worked on simple models with fewer than five variables, by adapting it for high-dimensional scenarios, accommodating categorical variables, and providing a PAC coverage guarantee. Our reweighting strategy based on the Random Forest algorithm can improve the PI computed using any nonconformity score. This results in more adaptive PI with marginal, training-conditional, and conditional coverage, making Conformal Predictive Intervals more similar to those produced by traditional statistics. This may ease their interpretation in terms of risks and give a clearer relationship between the length of the PI and the uncertainties of a given model $\widehat{\mu}$, thereby allowing for a better understanding of the limitations of $\widehat{\mu}$. 

\bibliography{References}
\bibliographystyle{plainnat}

\newpage 
\appendix

\begin{center}
    \huge \textbf{Supplementary Materials}
\end{center}
\hrulefill

\tableofcontents
\addtocontents{toc}{\protect\setcounter{tocdepth}{2}}

\newpage
\section{Proof of Lemma \ref{chap7:lemma:lcp}}
The Lemma \ref{chap7:lemma:lcp}, which is the cornerstone of the LCP framework, shows how to achieve marginal coverage by properly selecting the level $\alphat$ of the quantile of the localizer.

\begin{lemma} \label{suptheorem:lcp}
    Let $\alphat$ be the smallest value in $\Gamma = \left\{ \sum_{j=1}^{k} w_n(\X_i, \X_j): i=1, \dots, n;\; k=1, \dots, n\right\}$ such that 
    \begin{align} \label{eq:lcp_marg_sup}
        \sum_{i=1}^{n+1}\frac{1}{n+1} \mathds{1}_{\hV_i \leq \mathcal{Q}(\alphat;\;  \mathcal{F}_i )} \geq 1 - \alpha,
    \end{align}
\end{lemma}

then $\mathbb{P}\Bigl\{\hV_{n+1} \leq \mathcal{Q}(\alphat; \; \mathcal{F}_{n+1})\Bigl\} \geq 1-\alpha$, or equivalently $\mathbb{P}\Bigl\{\hV_{n+1} \leq \mathcal{Q}(\alphat; \; \mathcal{F})\Bigl\} \geq 1-\alpha$.

It is important to keep in mind that both $\alphat$ and $\mathcal{F}_{n+1}$ depends on $\widehat{\mathcal{D}}_{n} = \left\{\widehat{Z}_1, \dots, \widehat{Z}_{n} \right\}$ and $(\Xtest, \hV_{n+1})$ where $\widehat{Z}_i = (\X_i, \hV_i)$, but we will not specify them for ease of reading.

\begin{proof}
Let define the event $E_{n+1} = \left\{\widehat{Z}_1 = \widehat{z}_1, \dots, \widehat{Z}_{n+1} = \widehat{z}_{n+1}\right\}$ where $\widehat{Z}_i = (\X_i, \hV_{i})$ and $\widehat{z}_i = (\x_i, \widehat{v}_i) \in \mathcal{X} \times \mathbb{R}$. The exchangeability of the residuals implies that $\hV_{n+1}|E_{n+1}$ is uniform on the set  $\left\{\widehat{v}_1, \dots,\widehat{v}_{n+1}\right\}$, and
\begin{align*}
        P^{n+1}\left\{ \hV_{n+1} \leq \mathcal{Q}\left(\alphat;\; \mathcal{F}_{n+1}\right) \;\Big|\; E_{n+1}\right\} & = \sum_{i=1}^{n+1} P^{n+1}(\widehat{V}_{n+1} = \widehat{v}_i \;|\; E_{n+1}) \mathds{1}_{\widehat{v}_i \leq \mathcal{Q}(\alphat; \; \mathcal{F}_i) } \\
        & = \sum_{i=1}^{n+1} \frac{1}{n+1} \mathds{1}_{\widehat{v}_i \leq \mathcal{Q}(\alphat; \; \mathcal{F}_i) } \geq 1 - \alpha \quad \text{(by construction,  eq. \ref{eq:lcp_marg_sup})}
    \end{align*}

By marginalizing over the event $E_{n+1}$, we have $\mathbb{P}\Bigl\{\hV_{n+1} \leq \mathcal{Q}(\alphat; \; \mathcal{F}_{n+1})\Bigl\} \geq 1-\alpha$. Additionally, we can remove the dependence on the unknown residuals $\hV_{n+1}$ using the well-known fact that $\hV_{n+1} \leq \mathcal{Q}(\alphat; \; \mathcal{F}_{n+1}) \iff \hV_{n+1} \leq \mathcal{Q}(\alphat; \; \mathcal{F})$ (see lemma 1 of \citep{tibshirani2019conformal}] or Lemma A.1 of \citep{guanlocalizer}] for proof). Thus, we also have  $\mathbb{P}\Bigl\{\hV_{n+1} \leq \mathcal{Q}(\alphat; \; \mathcal{F})\Bigl\} \geq 1-\alpha$.
\end{proof}

We refer to the original paper \cite{guanlocalizer} for the proof of Theorem \ref{chap7:lemma:lcp} and Lemma \ref{chap7:eq:lemma_sk}.

\section{Proof of Theorem \ref{chap7:theo:pac_interval}: Training-conditional of LCP-RF}
In this section, we prove Theorem \ref{chap7:theo:pac_interval} that shows how to correct the LCP-RF approach to have training-conditional coverage.

\begin{theorem}
    Let $ \epsilon>0$,  $\alpha - \epsilon >0$ and $ \textcolor{red}{\widehat{\alpha}}$ be the smalest value in $T = \{\alpha_0=0,\dots, \alpha_K=\alpha\}$  s.t.
    \begin{align}
         \sum_{i=1}^{n_2} \frac{1}{n_2} \mathds{1}_{\hV^{2}_{i} \leq  \mathcal{Q}\left( \alphat(\X^{2}_{i}) + \textcolor{red}{\widehat{\alpha}};\; \mathcal{F}^{2, \infty}_{i}\right)} \geq 1 - \alpha.
    \end{align}
    Then, we have  $ P^{n_1}\Bigg\{ \cov(\mathcal{D}_{n_1}) \geq 1 - \alpha - \epsilon \Bigg\} \geq 1 - \delta$ with $\delta=K \exp(-2n_2\epsilon^2)$ and $ \cov(\mathcal{D}_{n_1}) = P \left\{ \hV_{n+1} \leq \mathcal{Q}\left(\alphat(\X_{n + 1}) + \textcolor{red}{\widehat{\alpha}};\; \mathcal{F}^{\infty}_{n+1}\right)\;\Big|\; D_{n_1}\right\}$.
\end{theorem}
\textbf{Remark.} This result is valid under the i.i.d assumption and not under exchangeability as the other results of the paper. We suggest choosing a grid $T \subset [0, \alpha]$ as we have observed in most practical scenarios that $\alphat(\X_{n + 1}) \approx 1 - \alpha$.  However, the central idea remains unaltered - to select a grid that enables transitioning from $\alphat(\X_{n+1})$ to 1. Additionally, as $\alphat(\X_{n + 1}) + \textcolor{red}{\widehat{\alpha}}$ may be above $1$, we define $\alphat(\X_{n + 1}) + \textcolor{red}{\widehat{\alpha}} := (\alphat(\X_{n + 1}) + \textcolor{red}{\widehat{\alpha}}) \lor 1$.

\begin{proof}
    Recall that here $\alphat$ and $\mathcal{F}^{\infty}_{n+1}$ is a function of $\widehat{\mathcal{D}}_{n_1} = \left\{\widehat{Z}_1, \dots, \widehat{Z}_{n_1} \right\}$ and $\Xtest$ as the RF has been trained on $\widehat{\mathcal{D}}_{n_1}$, but we will not specify $\widehat{\mathcal{D}}_{n_1}$ for ease of reading.
    \begin{align*}
        &P^{n_1}\Bigg\{ P\left\{ \hV_{n+1} \leq \mathcal{Q}\left(\alphat(\Xtest) + \textcolor{red}{\widehat{\alpha}};\; \mathcal{F}_{n+1}^{\infty}\right)\;\Big| D_{n_1}\right\}  \leq 1 - \alpha - \epsilon \Bigg\} \\
        & \leq P^{n_1}\Bigg\{ P\left\{ \hV_{n+1} \leq \mathcal{Q}\left(\alphat(\Xtest) + \textcolor{red}{\widehat{\alpha}};\; \mathcal{F}_{n+1}^{\infty}\right)\;\Big| D_{n_1}\right\}  \leq \sum_{i=1}^{n_2} \frac{1}{n_2} \mathds{1}_{\hV^{2}_{i} \leq  \mathcal{Q}\left( \alphat(\X^{2}_{i}) + \textcolor{red}{\widehat{\alpha}};\; \mathcal{F}^{2, \infty}_{i}\right)} - \epsilon \Bigg\} \\
        & = \mathbb{E} \Bigg[ P^{n_1}\Bigg\{ P\left\{ \hV_{n+1} \leq \mathcal{Q}\left(\alphat(\Xtest) + \textcolor{red}{\widehat{\alpha}};\; \mathcal{F}_{n+1}^{\infty}\right)\;\Big| D_{n_1}\right\}  \leq \sum_{i=1}^{n_2} \frac{1}{n_2} \mathds{1}_{\hV^{2}_{i} \leq  \mathcal{Q}\left( \alphat(\X^{2}_{i}) + \textcolor{red}{\widehat{\alpha}};\; \mathcal{F}^{2, \infty}_{i}\right)} - \epsilon \Bigg\} \Bigg| \mathcal{D}_{n_1}\Bigg\}  \Bigg] \\
        & \leq \sum_{\alpha \in T} \mathbb{E} \Bigg[ P^{n_1}\Bigg\{ P\left\{ \hV_{n+1} \leq \mathcal{Q}\left(\alphat(\Xtest) + \alpha;\; \mathcal{F}_{n+1}^{\infty}\right)\;\Big| D_{n_1}\right\}  \leq \sum_{i=1}^{n_2} \frac{1}{n_2} \mathds{1}_{\hV^{2}_{i} \leq  \mathcal{Q}\left( \alphat(\X^{2}_{i}) + \alpha;\; \mathcal{F}^{2, \infty}_{i}\right)} - \epsilon \Bigg\} \Bigg| \mathcal{D}_{n_1}\Bigg\}  \Bigg] \\
    \end{align*}
    Note that conditionally on $\mathcal{D}_{n_1}$, $\sum_{i=1}^{n_2} \frac{1}{n_2} \mathds{1}_{\hV^{2}_{i} \leq  \mathcal{Q}\left( \alphat(\X^{2}_{i}) + \alpha;\; \mathcal{F}^{2, \infty}_{i}\right)}$ is the average of $n_2$ bernoulli-trial with mean $P\left\{ \hV_{n+1} \leq \mathcal{Q}\left(\alphat(\Xtest) + \textcolor{red}{\widehat{\alpha}};\; \mathcal{F}_{n+1}^{\infty}\right)\;\Big| D_{n_1}\right\}$, therefore we can bound the conditional probability by using Hoeffding's inequality. Finally, we have

    \begin{align*}
        & P^{n_1}\Bigg\{ P\left\{ \hV_{n+1} \leq \mathcal{Q}\left(\alphat(\Xtest) + \textcolor{red}{\widehat{\alpha}};\; \mathcal{F}_{n+1}^{\infty}\right)\;\Big| D_{n_1}\right\}  \leq 1 - \alpha - \epsilon \Bigg\} \\
        & \leq \sum_{\alpha \in T} \mathbb{E} \Bigg[ P^{n_1}\Bigg\{ P\left\{ \hV_{n+1} \leq \mathcal{Q}\left(\alphat(\Xtest) + \alpha;\; \mathcal{F}_{n+1}^{\infty}\right)\;\Big| D_{n_1}\right\}  \leq \sum_{i=1}^{n_2} \frac{1}{n_2} \mathds{1}_{\hV^{2}_{i} \leq  \mathcal{Q}\left( \alphat(\X^{2}_{i}) + \alpha;\; \mathcal{F}^{2, \infty}_{i}\right)} - \epsilon \Bigg\} \Bigg| \mathcal{D}_{n_1}\Bigg\}  \Bigg] \\
        & \leq K \exp(-2 \epsilon^2 n_2)
    \end{align*}
\end{proof}

\section{Proof of marginal coverage of groupwise LCP-RF} \label{sup:group_marginal_coverage}
Here, we show that there is no loss in coverage guarantee when conformalizing by groups. We demonstrate the case of marginal coverage, the groupwise training-conditional is obtained in a similar way.
\begin{theorem} \label{chap7:group_lcp}
 Given a partition of the calibration data $\mathcal{D}_n$ in $G_1, \dots, G_L$ and their associated regions $\textbf{R}= \{R_1, \dots, R_L\}$ defined by the weighted adjacency matrix $w_n(\X_i, \X_j)$ of the RF, we denote $R(\X) \in \textbf{R}$ the region where $\X$ falls. At $\hV_{n+1} = v$, let define $\alphat(v, R(\Xtest))$ to be the smallest value $\alphat \in \Gamma$ such that
    \begin{align} \label{eq:lcp}
        \mathlarger{\sum}_{i \; :\; \X_i \in R(\Xtest)} \frac{1}{|\textbf{R}(\Xtest)|+1} \mathds{1}_{\hV_i \leq \mathcal{Q}(\alphat; \; \mathcal{F}^v_i)}  \geq 1 - \alpha.
    \end{align}
    Set $C_V(\X_{n+1}) = \left\{ v: v \leq \mathcal{Q}\left(\alphat\left(v, R(\Xtest)\right); \; \mathcal{F}\right)\right\}$, then $P^{n+1}\left(\hV_{n+1} \in C_V(\Xtest)\right) \geq 1 - \alpha$.
\end{theorem}

\begin{proof}
    \begin{align*}
        P\left(\hV_{n+1} \in C_V(\Xtest) \right) & = P\left( \hV_{n+1} \leq \mathcal{Q} \left(\alphat(\hV_{n+1}, R(\Xtest)); \; \mathcal{F}) \right) \right) \\
        & = \sum_{l=1}^{L} P(R_l) \; P\left( \hV_{n+1} \leq \mathcal{Q} \left(\alphat(\hV_{n+1}, R(\Xtest); \; \mathcal{F}\right) \;\big|\; \Xtest \in R_l\right) \\
        & \geq \sum_{l=1}^{L} P(R_l) (1 - \alpha) \quad \text{(by lemma \ref{suptheorem:lcp})} \\
        & \geq 1 - \alpha
    \end{align*}
\end{proof}

\section{Proof of Theorem \ref{chap7:theo:lcp_cond}: asymptotic conditional coverage}
Here, we prove the asymptotic conditional coverage of the LCP-RF approach or Theorem \ref{chap7:theo:lcp_cond}. Our primary contribution is lemma \ref{lemma:born}, which enables us to control the weights of the RF and, subsequently, to proceed with \cite{guanlocalizer}'s proof.

\begin{theorem} \label{theo:lcp_cond}
    Let $\alphat$ and $C_V(\Xtest)$ define as in Theorem \ref{chap7:lemma:lcp}. Under assumptions \ref{chap7:prop:assym1}-\ref{chap7:prop:assym3}, for all $\epsilon > 0$, we have $\small\lim_{n \rightarrow \infty} P^{n+1}\left( \hV_{n+1} \in C_V(\Xtest) \; |\;  \Xtest\right) = 1-\alpha$ and $\small lim_{n \rightarrow \infty} P^{n+1}\left( \max_v|\alphat(v) - (1-\alpha)| < \epsilon\; |\;  \Xtest\right) = 1$ a.s.
\end{theorem}

The bootstrap step in Random Forest makes its theoretical analysis difficult, which is why it has been replaced by subsampling without replacement in most studies that investigate the asymptotic properties of Random Forests \citep{scornet2015consistency, wager2017estimation, goehry2020random}. To circumvent this issue, we will use Honest Forest as a theoretical surrogate. Honest Forest is a variation of random forest that is simpler to analyze, and \cite{elie2020random} have shown that asymptotically, the original forest and the honest forest are close a.s. (see Lemma 6.2 in \citep{elie2020random}), thus we can extend the results from the Honest Forest to the original forest.

The main idea is to use a second independent sample $\mathcal{D}_n^\diamond=\{(\X^\diamond_i, Y^\diamond_i)\}_{i=1}^n$. We assume that we have a honest forest \citep{wager2017estimation} of the Quantile Regression Forest $\widehat{F}^\diamond(r | \X = \x, \Theta_1, \dots, \Theta_k, \mathcal{D}_n, \mathcal{D}_n^{\diamond})$, which is a random forest that is grown using $\mathcal{D}_n$, but uses another sample $\mathcal{D}_n^{\diamond}$ (independent of $\mathcal{D}_n$ and $\Theta_{1:k}$) to estimate the weights and the prediction. The Honest QRF is defined as:
\begin{equation*}
    F^{\diamond}(r | \X = \x, \Theta_1, \dots, \Theta_k, \mathcal{D}_n,  \mathcal{D}_n^{\diamond}) = \sum_{i=1}^{n+1}  w_n(\x, \X_j^\diamond)) \mathds{1}_{V^{\diamond i} \leq r}  
\end{equation*}
where $\Xtest^\diamond = \Xtest$ the test observation of interest and
\begin{equation*}
      w_n(\x, \X_j^\diamond) = \frac{1}{k} \sum_{l=1}^{k}  \frac{\mathds{1}_{\X_i^\diamond \in  A_{n}(\x; \; \Theta_l)}}{N^\diamond(A_{n}(\x; \Theta_l))},    
\end{equation*}
and $N^\diamond(A_{n}(\x; \Theta_l))$ is the number of observation of $\mathcal{D}_n^\diamond = \left\{(X_1^\diamond, \hV_1^\diamond) \dots, (X_n^\diamond, \hV_n^\diamond) \right\}\cup(\Xtest, \hV_{n+1})$ that fall into $A_{n}(\x; \; \Theta_l)$. To ease the notations, we do not write $\Theta_1, \dots, \Theta_k, \mathcal{D}_n,  \mathcal{D}_n^{\diamond}$ if not necessary we write $\widehat{F}^{\diamond}(r | \x)$ instead of $ \widehat{F}^{\diamond}(r | \X = \x, \Theta_1, \dots, \Theta_k, \mathcal{D}_n,  \mathcal{D}_n^{\diamond})$.

The following Lemma \ref{lemma:vapnik_boostrap} allows for control of the weights of the Honest Forest.

\begin{lemma} \label{lemma:vapnik_boostrap}
Consider $\mathcal{D}_n, \mathcal{D}_n^\diamond,$ two independent datasets of independent n samples of $(\boldsymbol{X}, Y)$. Build a tree using $\mathcal{D}_n$ with bootstrap and bagging procedure driven by $\Theta$. As before,  $N(A_{n}(\x; \Theta_l))$ is the number of bootstrap observations of $\mathcal{D}_n$ that fall into $A_{n}(\x; \; \Theta_l)$ and $N^\diamond(A_{n}(\x; \Theta_l))$ is the number of observations of  $\mathcal{D}_n^\diamond$ that fall into $A_{n}(\x; \; \Theta_l)$. Then:
\begin{align}
    \forall \epsilon >0, \quad \mathbb{P}\left(\labs N(A_{n}(\x; \Theta_l)) - N^\diamond(A_{n}(\x; \Theta_l)) \rabs > \epsilon\right) \leq 24(n+1)^{2p}e^{-\epsilon^2/288n}
\end{align}
\end{lemma}

See the proof in \citep{elie2020random}, Lemma 5.3.

The following Lemma is the key element to prove Theorem \ref{theo:lcp_cond} for Random Forest Localizer. 

\begin{lemma} \label{lemma:born}
    Let define for all $i=1, \dots, n$,
    \begin{align*}
        R_i = \sum_{j = 1}^{n} w_n(\X_i, \X_j^\diamond) \left( \mathds{1}_{\hV_j^\diamond < \hV_i} - F(\hV_i | \X^\diamond_j)\right) \quad \text{and} \quad I_i = \sum_{j = 1}^{n} w_n(\X_i, \X_j^\diamond) F(\hV_i | \X^\diamond_j),
    \end{align*}
    then for any $\epsilon >0$, under assumptions \ref{chap7:prop:assym1}-\ref{chap7:prop:assym3}, we have
    \begin{align}
    & P(|R_i|> \epsilon)  \leq  2(1 + 24 k (n+1)^{2p})\exp \left({\frac{K \ln(n)^{\beta}}{576 \sqrt{n}} - \frac{ \epsilon K \ln(n)^{\beta}}{24}}\right) \\
    & I_i \in \Big[F(\hV_i|\X_i) - v(n) - \frac{2 k}{K \sqrt{n} \ln(n)^\beta}, \quad F(\hV_i|\X_i) + v(n)+\frac{2 k}{K \sqrt{n} \ln(n)^\beta}\Big]
\end{align}
where $v(n)$ is a sequence s.t. $v(n) \xrightarrow[n \rightarrow \infty]{} 0$.

    \begin{proof}
        First, let's rewrite $R_i$ as
        \begin{align*}
            R_i = \sum_{j = 1}^{n} w_n(\X_i, \X_j^\diamond) \left( \mathds{1}_{\hV_j^\diamond < \hV_i} - F(\hV_i | \X_j^\diamond)\right) = \sum_{j = 1}^{n} w_n(\X_i, \X_j^\diamond) H_j^\diamond
        \end{align*}
        where $H_j^\diamond$ is bounded by $1$ and $E[H_j^\diamond | \X_j^\diamond, \mathcal{D}_n] = 0$. Then, for all $\epsilon>0$
        \begin{align*}
    P(R_i > \epsilon) & \leq e^{-t\epsilon} \; \mathbb{E}[e^{t R_i}] \\
    & \leq e^{-t\epsilon} \; \mathbb{E}\left[ \prod_{j = 1}^{n}  \mathbb{E} \left[e^{t   w_n(\X_i, \X_j^\diamond) H_j^\diamond} | \Theta_1, \dots, \Theta_k, \mathcal{D}_n, \X^\diamond_i \dots, \X^\diamond_{n} \right] \right] \\ 
    & \leq  e^{-t\epsilon} \; \mathbb{E}\left[ \prod_{j = 1}^{n}  e^{\frac{t^2}{2} w_n(\X_i, \X_j^\diamond)^2} \right]
\end{align*}
The last inequality comes from $w_n(\X_i, \X_j^\diamond)$ being constant given $\Theta_1, \dots, \Theta_k, \mathcal{D}_n, \X^\diamond_i \dots, \X^\diamond_{n}$, and as $H_j^\diamond$ is bounded by 1 with $E[H_j^\diamond | \X_j^\diamond, \mathcal{D}_n] = 0$, we used the following inequality: If $|X| \leq 1$ a.s and $\mathbb{E}[X] = 0$, then $\mathbb{E}[e^{t X}] \leq \mathbb{E}[e^{\frac{t^2}{2}}]$.

Using assumption \ref{chap7:prop:assym3}, there exists $K > 0$ such that $\forall l \in [k]$, $N(A_{n}(\X_i; \Theta_l)) \geq \frac{K \sqrt{n} \ln(n)^\beta}{2}$ a.s., then we have $\Gamma(l) = \{ N^\diamond(A_{n}(\X_i; \Theta_l)) \leq \frac{K \sqrt{n} \ln(n)^\beta}{2}\} \subset \{ |N(A_{n}(\X_i; \Theta_l))-N^\diamond(A_{n}(\X_i; \Theta_l))| \geq \frac{K \sqrt{n} \ln(n)^\beta}{2} \}$. Thus, using Lemma \ref{lemma:vapnik_boostrap}, we have that $\mathbb{P}(\Gamma(l)) \leq 24(n+1)^{2 p}\exp(-\frac{-K^2 (\ln(n)^{2\beta})}{1152})$. 

We have
\begin{align*}
    \sum_{j = 1}^{n} w_n(\X_i, \X_j^\diamond)^2 & = 
    \sum_{j = 1}^{n}  \frac{w_n(\X_i, \X_j^\diamond)}{k} \left( \sum_{l=1}^{k} \frac{\mathds{1}_{\X_j^\diamond \in A_{n}(\X_i; \Theta_l)}}{N^\diamond(A_{n}(\X_i; \Theta_l))} (\mathds{1}_{\{ \Gamma(l)\}} + \mathds{1}_{\{ \Gamma(l)^c\}})\right) \\
    & \leq \sum_{j = 1}^{n}  w_n(\X_i, \X_j^\diamond)^2 \left( \frac{2}{K \sqrt{n} \ln(n)^\beta} + \frac{1}{k} \sum_{l=1}^{k} \mathds{1}_{\X_j^\diamond \in A_{n}(\X_i; \Theta_l, \mathcal{D}_m)} \mathds{1}_{\{ \Gamma(l)\}}\right)
\end{align*}

So that, \begin{align*}
    P(R_i > \epsilon) & \leq \exp(-t\epsilon + \frac{t^2}{K \sqrt{n} \ln(n)^\beta}) \mathbb{E}\left[\exp\left(\frac{t^2}{2} \mathds{1}_{\cup_{l=1}^{k} \Gamma(l)}\right)\right] \\
    & \leq \exp(-t\epsilon + \frac{t^2}{K \sqrt{n} \ln(n)^\beta}) \times \left( 1 + e^{\frac{t^2}{2}} \sum_{l=1}^{k} \mathbb{P}(\Gamma(l)) \right) \\
    & \leq \exp(-t\epsilon + \frac{t^2}{K \sqrt{n} \ln(n)^\beta}) \times \left( 1 + 24 k (n+1)^{2p}\exp \left({\frac{t^2}{2} - \frac{K^2 \ln(n)^{2\beta}}{1152}}\right)  \right)
\end{align*}
Taking $t^2 = \frac{K^2 \ln(n)^{2\beta}}{576}$ leads to 
\begin{align*}
    P(R_i> \epsilon) & \leq  (1 + 24 k (n+1)^{2p})\exp \left({\frac{K \ln(n)^{\beta}}{576 \sqrt{n}} - \frac{ \epsilon K \ln(n)^{\beta}}{24}}\right) 
\end{align*}

We obtain the same bound for $\mathbb{P}(R_i \leq -\epsilon) = \mathbb{P}(-R_i > \epsilon)$, then by using assumption \ref{chap7:prop:assym2}, there exists  $k = \mathcal{O}(n^\alpha)$ so that the right term is finite, we conclude by Borel-Cantelli that $|R_i|$ goes to 0 a.s. Finally, we have  
\begin{align*}
    P(|R_i|> \epsilon) & \leq  2(1 + 24 k (n+1)^{2p})\exp \left({\frac{K \ln(n)^{\beta}}{576 \sqrt{n}} - \frac{ \epsilon K \ln(n)^{\beta}}{24}}\right) 
\end{align*}

Now, we consider $I_i$. By assumption \ref{chap7:prop:assym2}, we have $\forall x \in \mathbb{R}^d, \forall r \in \mathbb{R}, \sup_{\boldsymbol{z} \in A_n(\boldsymbol{x}; \;\Theta_l)} |F(r | \boldsymbol{\boldsymbol{z}}) - F(r|\boldsymbol{x})| \overset{a.s}{\to} 0$ then we can assume that there exists a sequence $v(n) \rightarrow 0$ s.t.  
\begin{align}
    \forall x \in \mathbb{R}^d, \forall r \in \mathbb{R}, \sup_{\boldsymbol{z} \in A_n(\boldsymbol{x}; \;\Theta_l)} |F(r | \boldsymbol{\boldsymbol{z}}) - F(r|\boldsymbol{x})| \leq v(n)
\end{align}
Consequently, 
\begin{align*}
    |I_i - F(\hV_i | \X_i)| & = \big|\sum_{j = 1}^{n+1} w_n(\X_i, \X_j^\diamond)\left( F(\hV_i | \X_j^\diamond) - F(\hV_i | \X_i) \right) - w_n(\X_i, \Xtest) F(\hV_i | \X_{n+1}^\diamond)\big|\\
    & \leq \sum_{j = 1}^{n+1} w_n(\X_i, \X_j^\diamond) \big|F(\hV_i | \X_j^\diamond) - F(\hV_i | \X_i)\big| + w_n(\X_i, \Xtest) \\
    & \leq \sum_{j = 1}^{n+1} w_n(\X_i, \X_j^\diamond)  \sup_{\boldsymbol{z} \in A_n(\X_i; \Theta_l)}\big|F(\hV_i | \boldsymbol{z}) - F(\hV_i | \X_i)\big| + \frac{2 k}{K \sqrt{n} \ln(n)^\beta} \\
    & \leq v(n) + \frac{2 k}{K \sqrt{n} \ln(n)^\beta}
\end{align*}    
We use the fact that by assumption \ref{chap7:prop:assym3}, we can lower bound the weights of the forest since $N(A_{n}(\X_i; \Theta_l)) \geq \frac{K \sqrt{n} \ln(n)^\beta}{2}$ for all $l \in [k]$, we have $ w_n(\X_i, \Xtest) \leq \frac{2 k}{K \sqrt{n} \ln(n)^\beta}$.
\end{proof}
\end{lemma}

\subsection{Proof of Theorem \ref{theo:lcp_cond}}

As in \cite{guanlocalizer}, we first prove that $\alphat \rightarrow 1 - \alpha$ and then show that the resulting PI has coverage rate that achieves the desired level $1 - \alpha$.

\begin{proof}
Let define $R_i = \sum_{j = 1}^{n} w_n(\X_i, \X_j^\diamond) \left( \mathds{1}_{\hV_j^\diamond < \hV_i} - F(\hV_i | \X^\diamond_j)\right)$ and $I_i = \sum_{j = 1}^{n} w_n(\X_i, \X_j^\diamond) F(\hV_i | \X^\diamond_j)$ for all $i=1, \dots, n$, and 
    \begin{align*}
        J_i(v, \alphat) := \left\{ \hV_i \leq \mathcal{Q}(\alphat; \; \mathcal{F}^v_i) \right\} = \left\{ \alphat > \sum_{j \leq n: \hV_j < \hV_i} w_n(\X_i,\Xd_j) + w_n(\X_i, \Xd_{n+1})\mathds{1}_{v < \hV_i}\right\}
    \end{align*}
which is the event in the sum defined in theorem \ref{chap7:lemma:lcp}. Let consider the lower term in $J_i(v, \alphat)$. We have
\begin{multline} \label{eq:ji_born}
     \small R_i + I_i - w_n(\X_i, \Xd_{n+1}) \leq R_i + I_i \leq \sum_{j \leq n: \hV_j < V_i} w_n(\X_i,\Xd_j) + w_n(\X_i, \Xd_{n+1})\mathds{1}_{v < \hV_i}  \\
     \leq R_i + I_i + w_n(\X_i, \Xd_{n+1})
\end{multline}

Let $\epsilon > 0$, and denote $ G = \left\{i \in \{1, \dots, n\}: |R_i| \leq \epsilon \right\}$. By lemma \ref{lemma:born}, we have $ I_i \in \Big[F(\hV_i|\X_i) - v(n) - \frac{2 k}{K \sqrt{n} \ln(n)^\beta}, \quad F(\hV_i|\X_i) + v(n)+\frac{2 k}{K \sqrt{n} \ln(n)^\beta}\Big]$. Using the upper bound of equation \ref{eq:ji_born}, for any $i \in G$, we have
\begin{align}
    J^{down}_i(\alphat) := \left\{\alphat >  F(\hV_i|\X_i) + \epsilon + v(n) + \frac{4 k}{K \sqrt{n} \ln(n)^\beta}\right\} \subseteq J_i(v, \alphat) 
\end{align}
and similary with the lower bound of equation \ref{eq:ji_born}, we have 

\begin{align}
    J^{up}_i(\alphat) := \left\{\alphat >  F(\hV_i|\X_i) - \epsilon - v(n) - \frac{4 k}{K \sqrt{n} \ln(n)^\beta}\right\} \supseteq J_i(v, \alphat) 
\end{align}
Hence, we can upper and lower bound the left side of the equation in Theorem \ref{chap7:lemma:lcp} using $J^{up}_i(\alphat)$ and $J^{down}_i(\alphat)$.
\begin{align}
    & \frac{1}{n+1} \sum_{i=1}^{n+1} J_i(v, \alphat) \leq \frac{1}{n+1} + \frac{1}{n+1} \sum_{i \in G}  J^{up}_i(\alphat) + \frac{|\Bar{G}|}{n+1} \label{eq:1}\\
    & \frac{1}{n+1} \sum_{i=1}^{n+1} J_i(v, \alphat) \geq \frac{1}{n+1} \sum_{i \in G}  J^{down}_i(\alphat)\label{eq:2} 
\end{align}
$W_i = F(\hV_i | \X_i)$ is an i.i.d. uniform distribution as $\hV | \X_i$ is a continuous random variable. Therefore, on the event $\{ |\Bar{G}| = 0 \}$, if $\alphat$ satistfy the marginal coverage of equation of theorem \ref{chap7:lemma:lcp}, then 

\begin{itemize}
    \item By equation \ref{eq:1}, we have
    \begin{multline}
       \small  \frac{1}{n+1}\left(1 + \sum_{i=1}^{n} J^{up}_i(\alphat)\right) \geq 1 - \alpha \\ \implies \alphat \geq \mathcal{Q}\left( \frac{n+1}{n}(1 - \alpha) - \frac{1}{n}; \; \frac{1}{n} \sum_{i=1}^{n} W_i\right) - \epsilon - v(n) - \frac{4 k}{K \sqrt{n} \ln(n)^\beta}
    \end{multline}
    The implication cames from the fact that $\frac{1}{n+1}\left(1 + \sum_{i=1}^{n} J^{up}_i(\alphat)\right) \geq 1 - \alpha$ implies that at lest $\lceil (n+1)(1-\alpha)\rceil - 1$ of the $J^{up}_i(\alphat) := \left\{\alphat >  F(\hV_i|\X_i) - \epsilon - v(n) - \frac{4 k}{K \sqrt{n} \ln(n)^\beta}\right\}$ is true. Assuming $J^{up}_i(\alphat)$ is true, and replacing $W_i = F(\hV_i|\X_i)$  then the order statistics $W_{(\lceil (n+1)(1-\alpha)\rceil -1)}$ should also satisfy the condition by definition. Note that $ \mathcal{Q}\left( \frac{n+1}{n}(1 - \alpha) - \frac{1}{n}; \; \frac{1}{n} \sum_{i=1}^{n} W_i\right) = W_{(\lceil (n+1)(1-\alpha)\rceil -1)}$.
    \item Similary, with equation \ref{eq:2}, we have
    \begin{multline*}
        \frac{1}{n+1}\sum_{i=1}^{n} J^{down}_i(\alphat) \geq 1 - \alpha \\
        \implies \alphat \geq \mathcal{Q}\left( \frac{n+1}{n}(1 - \alpha); \; \frac{1}{n} \sum_{i=1}^{n} W_i\right) + \epsilon + v(n) + \frac{4 k}{K \sqrt{n} \ln(n)^\beta}
    \end{multline*}
\end{itemize}
The maximal deviation between two weights of the forest is $\frac{4 k}{K \sqrt{n} \ln(n)^\beta}$, then there exists $C$ s.t
\begin{align}
    \alphat \leq \mathcal{Q}\left( \frac{n+1}{n}(1 - \alpha); \; \frac{1}{n} \sum_{i=1}^{n} W_i\right) + \epsilon + v(n) + \frac{4 k C}{K \sqrt{n} \ln(n)^\beta}
\end{align}
Therefore, on the event $\{ |\Bar{G}| = 0\}$, we have
\begin{multline*}
  \small \mathcal{Q}\left( \frac{n+1}{n}(1 - \alpha) - \frac{1}{n}; \; \frac{1}{n} \sum_{i=1}^{n} W_i\right) - \epsilon - v(n) - \frac{4 k}{K \sqrt{n} \ln(n)^\beta} \leq  \alphat \\
  \small \leq \mathcal{Q}\left( \frac{n+1}{n}(1 - \alpha); \; \frac{1}{n} \sum_{i=1}^{n} W_i\right) + \epsilon + v(n) + \frac{4 k C}{K \sqrt{n} \ln(n)^\beta}
\end{multline*}
In addition, let $\epsilon^\prime >0$ and consider the event $H =\left\{ \sup_t |\mathcal{Q}(t;\;  \frac{1}{n} \sum_{i=1}^{n} W_i) - t| \leq \epsilon^\prime\right\}$, on this even we have
\begin{multline} 
   \small (1 - \alpha) + \frac{1}{n}(1 - \alpha) - \frac{1}{n}  - \epsilon^\prime - \epsilon - v(n) - \frac{4 kC}{K \sqrt{n} \ln(n)^\beta} \leq  \alphat \\
   \leq (1 - \alpha) + \frac{1}{n}(1 - \alpha) + \epsilon^\prime + \epsilon + v(n) + \frac{4 k C}{K \sqrt{n} \ln(n)^\beta}
\end{multline}
Then, there exists C and $\epsilon$ s.t.
\begin{align}\label{eq:816}
   (1 - \alpha) - \epsilon - v(n) - \frac{4 k C}{K \sqrt{n} \ln(n)^\beta} \leq  \alphat \leq (1-\alpha)  + \epsilon + v(n) + \frac{4 k C}{K \sqrt{n} \ln(n)^\beta}
\end{align}
We can simplify the equation \ref{eq:816} as
\begin{align}
    |\alphat - (1-\alpha)| \leq \epsilon + v(n) + \frac{4 k C}{K \sqrt{n} \ln(n)^\beta}
\end{align}
Finally, we have for any $\hV_{n+1}=v$,

\begin{align*}
 P\left( |\alphat - (1-\alpha)| > \epsilon + v(n) + \frac{4 k C}{K \sqrt{n} \ln(n)^\beta}\right) & \leq P(\Bar{G}) + P(\Bar{H}) 
\end{align*}
Using DKW inequality \citep{massart1990tight} for $\Bar{H}$, and union bound for $\Bar{G}$, we have
\begin{align*}
    & P(\Bar{H}) = P(\sup_t |\mathcal{Q}(t;\; \frac{1}{n}\sum_{i=1}^{n} W_i) - t| > \epsilon^\prime) \leq 2 \exp(-2n\epsilon^{\prime 2}) \\
    & P(\Bar{G}) = P(\exists i \in \{1,\dots, n\}: |R_i| > \epsilon)  \leq n \times  2(1 + 24 k (n+1)^{2p})\exp \left({\frac{K \ln(n)^{\beta}}{576 \sqrt{n}} - \frac{ \epsilon K \ln(n)^{\beta}}{24}}\right)
\end{align*}
Consequently, we have $P\left( |\alphat - (1 - \alpha)| > \epsilon + v(n) + \frac{4 k C}{K \sqrt{n} \ln(n)^\beta}\right) \xrightarrow[n \rightarrow \infty]{} 0$ with $\epsilon + v(n) + \frac{4 k C}{K \sqrt{n} \ln(n)^\beta} \xrightarrow[n \rightarrow \infty]{} 0$ which conclude our proof. 

Now, let's prove that $\lim_{n \rightarrow \infty} P^{n+1}\left( \hV_{n+1} \in C_V(\Xtest) \; |\;  \Xtest\right) = 1-\alpha$ a.s. 

By definition, we have 
\begin{align} \label{eq:3}
\hV_{n+1} \leq \mathcal{Q}(\alphat; \; \mathcal{F}) \iff \sum_{i=1}^n w_n(\Xtest, \Xd_i) \mathds{1}_{\hV_i < \hV_{n+1}} = I_{n+1} + R_{n+1} < \alphat.
\end{align}

Let's denote $G = \{|R_{n+1}| \leq \epsilon \}$ with $\epsilon=\frac{1}{n}$. On the event $G$, we can lower and upper bound the left side of equation \ref{eq:3} using Lemma \ref{lemma:born} as above. As a result, we have: 

\begin{align}
    & I_{n+1} + R_{n+1} \leq F(\hV_{n+1}|\Xtest) + v(n) + \frac{4 k C}{K \sqrt{n} \ln(n)^\beta} + \epsilon \\
    & I_{n+1} + R_{n+1} \geq F(\hV_{n+1}|\Xtest) - v(n) - \frac{4 k C}{K \sqrt{n} \ln(n)^\beta} - \epsilon
\end{align}

Since $F(\hV_{n+1}|\Xtest)$ is an i.i.d uniform distribution, and $P(\Bar{G}) \rightarrow 0$, we have
\begin{align}
    & P(I_{n+1} + R_{n+1} < \alphat) \leq \alphat + v(n) + \frac{4 k C}{K \sqrt{n} \ln(n)^\beta} + \epsilon + P(\Bar{G}) \rightarrow \alphat\\
    & P(I_{n+1} + R_{n+1} < \alphat) \geq \alphat - v(n) - \frac{4 k C}{K \sqrt{n} \ln(n)^\beta} - \epsilon \rightarrow \alphat
\end{align}

As we have shown that $\alphat \rightarrow 1 - \alpha$ for any $\hV_{n+1}=v$ in probability, the LCP-RF achieve the asymptotic conditional coverage at level $1 - \alpha$.
\end{proof}

\newpage 
\section{Additional experiments}

In this section we present additional experiments on real-world datasets. First, we show the lengths and residuals of the PI when $\widehat{\mu}$ is a linear model with $\hV(\X, Y) = |Y - \widehat{\mu}(\X)|$ on bike datasets.

\begin{figure*}[ht!]
\centering
\subfigure[bike (lengths)]{\includegraphics[width=0.3\textwidth]{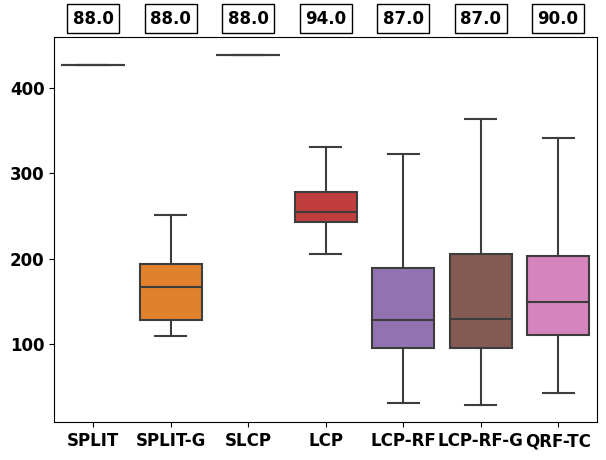}\label{fig:bike_lengths_lin}}
\subfigure[bike ($\widehat{err}_{n+1}$)]{\includegraphics[width=0.3\textwidth]{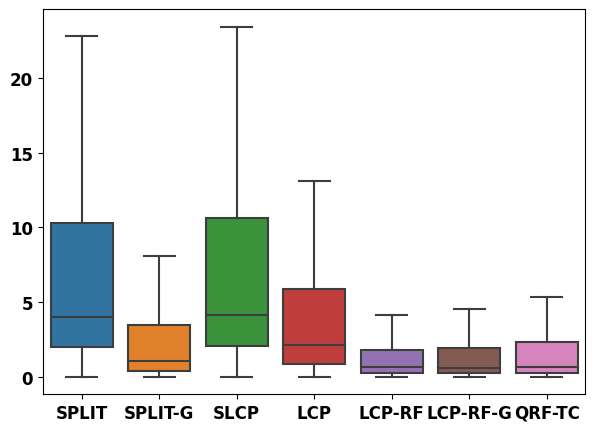} \label{fig:bike_residual_lin}}
\end{figure*}

Now, we run the experiment above on star and bike dataset using quantile score $\small V(\X, Y, \{\widehat{q}_{\alpha/2}, \widehat{q}_{1-\alpha/2} \})=\max\left(\widehat{q}_{\alpha/2}(\X)-Y, Y-\widehat{q}_{1-\alpha/2}(\X)\right)$. We first estimate $\{\widehat{q}_{\alpha/2}, \widehat{q}_{1-\alpha/2} \}$ using quantile linear regression \citep{chernozhukov2010quantile}, then we use Quantile Regression Forest. Note that in this case, split-CP corresponds to Conformalized Quantile Regression \citep{romano2019conformalized}.  

\begin{figure*}[ht!]
\centering
\subfigure[star (lengths) - QLR]{\includegraphics[width=0.22\textwidth]{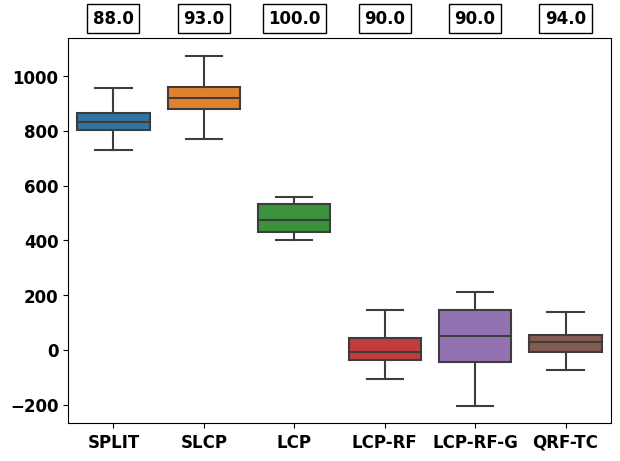}\label{fig:star_lql}}
\subfigure[star (residuals) - QLR]{\includegraphics[width=0.22\textwidth]{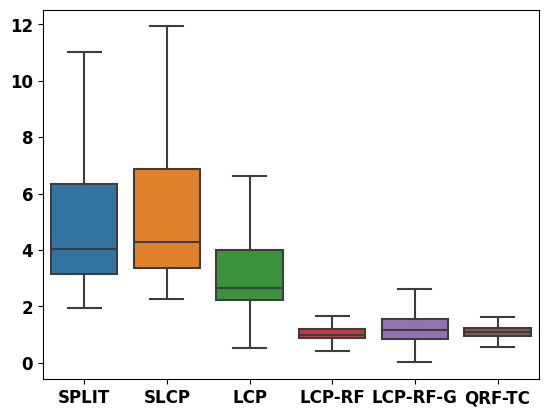} \label{fig:star_lqr}}
\subfigure[bike (lengths) - QLR]{\includegraphics[width=0.22\textwidth]{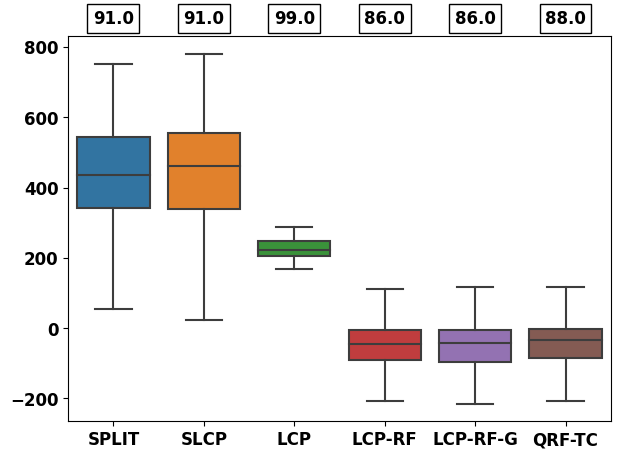}}
\subfigure[bike (residuals) - QLR]{\includegraphics[width=0.22\textwidth]{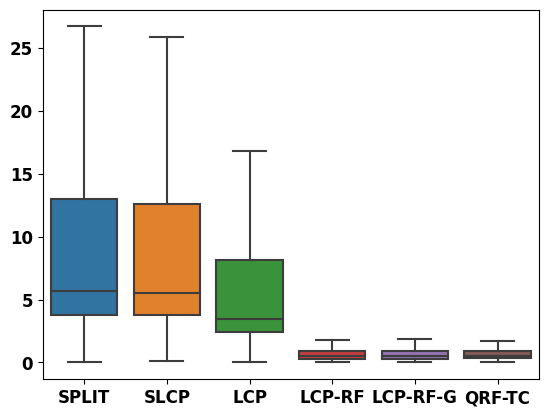}}
\caption{Lengths and errors distribution of quantile score using Quantile Linear Regression (QLR)}
\end{figure*}

We also compute the quantile score using Quantile Regression Forest in the figure below.

\begin{figure*}[ht!]
\centering
\subfigure[star (lengths) - QRF]{\includegraphics[width=0.22\textwidth]{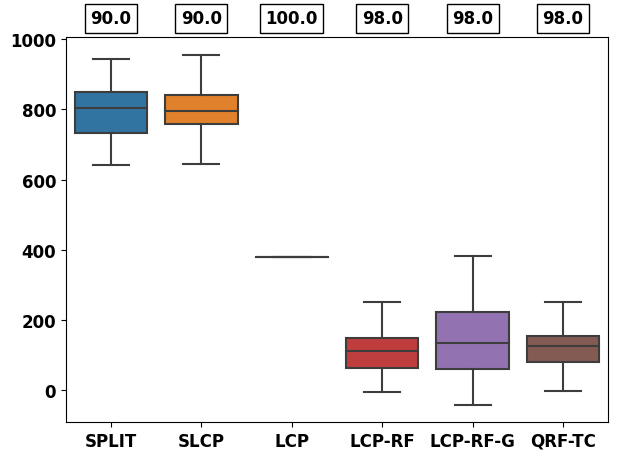}\label{fig:star_lql}}
\subfigure[star (residuals) - QRF]{\includegraphics[width=0.22\textwidth]{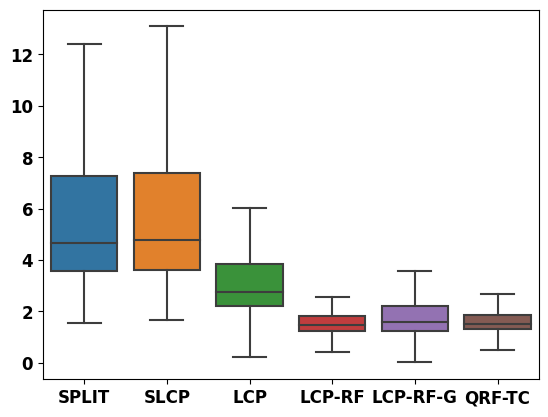} \label{fig:star_lqr}}
\subfigure[bike (lengths) - QRF]{\includegraphics[width=0.22\textwidth]{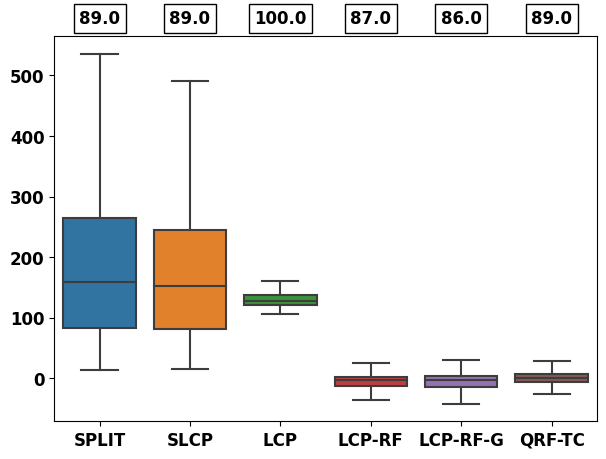}}
\subfigure[bike (residuals) - QRF]{\includegraphics[width=0.22\textwidth]{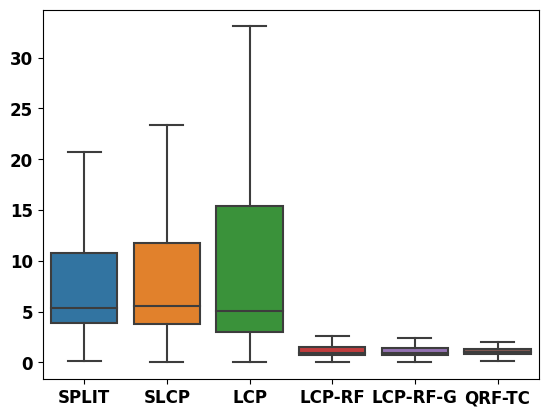}}
\caption{Lengths and errors distribution of quantile score using Quantile Random Forest (QRF)}
\end{figure*}
All these figures show that the Random Forest Localizer performs much better than the other methods.

\end{document}